\documentclass{article} 
\usepackage{iclr2027_conference,times}

\usepackage{amsmath,amsfonts,bm}

\def\eqref#1{equation~\ref{#1}}

\def\1{\bm{1}}

\DeclareMathAlphabet{\mathsfit}{\encodingdefault}{\sfdefault}{m}{sl}
\SetMathAlphabet{\mathsfit}{bold}{\encodingdefault}{\sfdefault}{bx}{n}

\usepackage{hyperref}       
\usepackage{url}            
\usepackage{booktabs}       
\usepackage{amsfonts}       
\usepackage{nicefrac}       
\usepackage{microtype}      
\usepackage{xcolor}         
\usepackage{graphicx}
\usepackage{algorithm}
\usepackage{algorithmic}
\usepackage{url}
\usepackage{color}
\usepackage{colortbl}
\usepackage{multirow}
\usepackage{multicol}
\usepackage{diagbox}
\usepackage{wrapfig}
\usepackage{amsmath,amssymb}
\usepackage{makecell}
\usepackage{appendix}
\usepackage{subcaption}
\usepackage{wrapfig}
\usepackage{amsthm,mathtools}

\newtheorem{proposition}{Proposition}

\theoremstyle{definition}

\theoremstyle{remark}

\usepackage{tcolorbox}
\usepackage{xcolor}
\definecolor{promptorange}{RGB}{204, 102, 0}   
\definecolor{promptred}{RGB}{180, 0, 0}         
\definecolor{paleblue}{RGB}{235,245,251}

\title{Back to the Definition: Estimating Step-Level Advantages via Trajectory Graphs for Agentic Reinforcement Learning}

\author{Xincheng Yao$^{1}$\thanks{Work done during internship at Tencent AI Platform Department.}, Haobo Fu$^{2}$, Weiming Liu$^{2}$, Chongyang Zhang$^{1,3}$\thanks{Corresponding Author.} \\
  $^1$School of Information Science and Electronic Engineering, Shanghai Jiao Tong University.\\
  $^2$Tencent AI Platform Department.\\
  $^3$MoE Key Lab of Artificial Intelligence, AI Institute, Shanghai Jiao Tong University.\\
  \texttt{\{i-dover, sunny\_zhang\}@sjtu.edu.cn$^1$} \\
  \texttt{\{haobofu, weimingliu\}@tencent.com$^2$} \\
}

\iclrfinalcopy 
\begin{document}

\maketitle

\begin{abstract}
Group-based reinforcement learning (RL) methods, such as GRPO and its variants, have become a leading paradigm for training reasoning and agentic large language models (LLMs). While their group-normalized advantage estimation is reliable at the response level, it becomes systematically biased at the step level, since coarse-grained trajectory-level advantages are hard to accurately reflect the contribution of individual steps (\emph{i.e}, failed trajectories may contain valuable steps). Revisiting the foundational RL definition, we notice that GRPO's success on single-turn tasks stems from its advantage estimation strategy, which adheres to the basic definition: the mean reward of multiple actions sampled from the same state constitutes a credible state-value estimate. Extending the faithful estimation to step-level would in principle demand sampling multiple actions from each intermediate state, which is too costly on a per-state basis. To mitigate this issue, we can aggregate similar states across trajectories to better leverage global information. Since each trajectory is a chain of state-action transitions, cross-trajectory information flow requires modeling state-action transitions across trajectories, which naturally forms a directed graph. Building on this insight, we propose a \textbf{Gra}ph-based \textbf{F}aithful s\textbf{T}ep-level credit-assignment framework (\textbf{GRAFT}) that grafts all rollout trajectories into a trajectory graph, recovering node state-values via Bellman iteration on the graph, and assigning credit to each edge by the node value difference. Theoretically, the estimated step-level advantage faithfully adheres to
the basic advantage definition in RL. To further ensure the reliability of step-level advantage estimation, we further propose Graph GAE, which extends GAE to the trajectory graph for reducing the impact of state-value estimation bias. Experiments across a range of multi-turn agentic benchmarks show consistent gains over GRPO and superior performance compared to recent agentic RL algorithms. Code will be available at \url{https://github.com/xcyao00/GRAFT}.
\end{abstract}

\section{Introduction}

Reinforcement learning (RL) has emerged as a cornerstone for training reasoning and agentic large language models (LLMs) \citep{OpenAI-o1, DeepSeek-R1, Kimi-2.5, qwen35blog, Tongyi-DeepSearch, qwen-agentworld}. Among the various RL algorithms \citep{PPO, rloo, DPO, REINFORCE, GRPO} adapted to the LLM domain, Group Relative Policy Optimization (GRPO) \citep{GRPO} and its variants \citep{DAPO, GSPO, SAPO, HTPO} have gained the most significant traction, primarily due to their critic-free design. For each prompt, GRPO samples a group of $N$ responses and estimates the advantage of each response by normalizing its reward against the group-level statistics. Within the standard RL framework, this strategy is theoretically well-grounded at the response level: when the entire response is treated as a single action and the prompt $q$ as the initial state, the group mean exactly serves as a valid estimate of the state value $V_\pi(q)$, and the resulting normalized advantage aligns with the standard definition $A_\pi(q,a) = Q_\pi(q,a) - V_\pi(q)$ (see Sec.\ref{sec:preliminary}).

However, this mechanism shifts fundamentally in multi-turn tasks, where trajectory-level advantages are too coarse-grained to accurately capture the contribution of individual steps. A successful trajectory may contain redundant or erroneous steps that receive unwarranted credit, while a failed trajectory may include valuable steps whose benefits are obscured by subsequent mistakes. A common remedy is to employ a Process Reward Model (PRM) to assign step-wise rewards and extend the group-based advantage estimation at each step index $t$ \citep{GRPO, prime, agentprm}. However, training PRMs requires costly step-level annotations and often suffers from distributional shift when applied to out-of-distribution reasoning traces. Moreover, as analyzed in Sec.~\ref{sec:preliminary}, applying group normalization per step $t$ incurs systematic bias. Because $s_{i,t}$ is a trajectory prefix, different trajectories usually occupy different intermediate states at the same step index $t$; consequently, the ``group'' used to estimate $V_\pi(s_{i,t})$ aggregates samples from distinct states rather than a single shared state. The resulting bias is systematic and grows with trajectory diversity. The theoretically faithful solution is to resample $N-1$ additional continuations from each intermediate state, but this recovers unbiasedness at an $\mathcal{O}(N^2 T)$ rollout cost that is essentially unaffordable for long-horizon multi-turn training.

To mitigate this issue while preserving theoretically faithful advantage estimation, we should aggregate similar states across trajectories to better leverage global information. Since each trajectory can be regarded as a chain of state-action transitions, cross-trajectory information flow requires modeling state-action transitions across trajectories, which
naturally forms a directed graph, where nodes represent states and edges represent actions. With this insight, we can reorganize $N$ trajectories from isolated chains into a unified trajectory graph. Each node in this graph aggregates all actions executed from the same state across all trajectories, providing the on-state action groups required for a more credible step-level advantage estimation at no additional cost. 

Building on the trajectory graph insight, we propose a \textbf{Gra}ph-based \textbf{F}aithful s\textbf{T}ep-level credit assignment (\textbf{GRAFT}) framework, which grafts $N$ rollout trajectories into a unified trajectory graph so that more faithful step-level advantages can be estimated in accordance with the foundational RL definition. First, \emph{graph construction} canonicalizes states via exact matching or embedding-based similarity matching, merging semantically identical states across trajectories into a single node with a shared on-state action group. Second, to estimate the state-value for every node, \emph{value propagation} applies Bellman iteration over the graph using only the rewards attached to the terminal nodes. In sparse reward settings, the converged value estimation is a direct empirical approximation of the definition of a standard value function. Third, given these node values, we define the step advantage as $A_t = \gamma V(s_{t+1}) - V(s_t)$, which coincides with the standard advantage definition (see Sec.\ref{sec:method}). Additionally, to further ensure the reliability of step-level advantage estimation, we take advantage of Generalized Advantage Estimation (GAE), which is the de facto in RL that can reduce the impact of state-value estimation bias. Specifically, we propose Graph GAE, which extends GAE to the trajectory graph by computing a weighted sum of the subsequent multi-hop average advantage estimates.  Compared to prior methods, the step-level advantage estimates obtained by our method are theoretically sound and exactly aligned with the fundamental definition. Furthermore, GRAFT only requires outcome rewards, eliminating the reliance on PRM-assigned step-level rewards. Experiments across a range of multi-turn agentic benchmarks show consistent gains over GRPO and superior performance compared to recent agentic RL algorithms.


\section{Preliminary Analysis}
\label{sec:preliminary}

In this section, we revisit the advantage estimation problem from the perspective of standard definitions in reinforcement learning. We begin by formally defining the state-action (Q) function $Q_\pi$, the value function $V_\pi$, and the advantage function $A_\pi$:
\begin{align}
\label{eq:definition}
    & Q_\pi(s_t, a_t) = \mathbb{E}_{s_{t+1}, a_{t+1}, \dots}\bigg[\sum_{l=0}^{\infty}\gamma^lr(s_{t+l})\bigg] 
    & V_\pi(s_t) = \mathbb{E}_{a_t, s_{t+1}, a_{t+1}, \dots}\bigg[\sum_{l=0}^{\infty}\gamma^lr(s_{t+l})\bigg] \nonumber \\  
    & A_\pi(s_t, a_t) = Q_\pi(s_t, a_t) - V_\pi(s_t) 
\end{align}
where $a_t \sim \pi(a_t \mid s_t)$ and $s_{t+1} \sim P(s_{t+1} \mid s_t,a_t)$. The $Q_\pi(s_t, a_t)$ denotes the expected return when taking an action $a_t$ at state $s_t$, $Q_\pi(s_t, a_t) = \mathbb{E}_{s_{t+1}, a_{t+1}, \dots}[r(s_t) + \gamma r(s_{t+1}) + \gamma^2r(s_{t+2}) + \dots]$\footnote{If the $s_{t+1}$ is deterministic, $Q_\pi(s_t, a_t)$ is also equal to $r(s_t) + \gamma V_\pi(s_{t+1})$.}, where $r(s_t)$ is the reward for transition ($s_t$, $a_t$, $s_{t+1}$). Compared to $Q_\pi(s_t, a_t)$, the only difference in $V_\pi(s_t)$ is that the action $a_t$ needs to be expected, which means the expected return starting from state $s_t$. Then, the connection between $Q_\pi(s_t, a_t)$ and $V_\pi(s_t)$ is that $V_\pi(s_t) = \mathbb{E}_{a_t}[Q_\pi(s_t, a_t)]$. Thus, the advantage can be understood as ``the benefit of performing a specific action $a_t$ in state $s_t$ compared to the average effect of all actions at that state''.

In LLMs, the input prompt $q$ can be regarded as the initial state, denoted as $s_1$. For action, it can be defined at various levels of granularity: as the whole response $o$, individual tokens, or intermediate reasoning steps. In GRPO \citep{GRPO}, the whole response is treated as a single action, for which a scalar reward is assigned. Since the generation process is terminated, there will be no subsequent $a_{t+1}$. Moreover, state transitions in LLMs are inherently deterministic (\emph{i.e.}, $s_{t+1}$ is obtained by concatenating the generated tokens $a_t$ to the previous context $s_t$). Therefore, the expectation over $s_{t+1}$ in estimating $Q_\pi(s_t, a_t)$ can be omitted, also the $a_{t+1}$ and subsequent states and actions can all be omitted. Thus, the Q-value of ($q$, $o$) simplifies to $Q(q, o) = r$. Accordingly, the value of $q$ can be defined as $V(q) = \mathbb{E}_o[Q(q, o)] = \frac{1}{N}\sum_{i=1}^{N}Q(q, o_i) = \frac{1}{N}\sum_{i=1}^{N}r_i$. 

We now revisit the advantage estimation strategy in GRPO, where the advantage is calculated as:
\begin{equation}
\label{eq:grpo_advantage}
    \hat{A}_i = \frac{r_i - {\rm mean}(\{r_1, r_2, \dots, r_N\})}{{\rm std}(\{r_1, r_2, \dots, r_N\})}
\end{equation}
It can be found that this estimator aligns precisely with the standard definition of the advantage function in Eq.(\ref{eq:definition}), as $r_i = Q(q, o_i)$ and ${\rm mean}(\{r_i, r_2, \dots, r_N\}) = V(q)$ (GRPO has additional normalization). However, a critical mismatch arises: while the derivation treats the entire response as a single action, GRPO assigns $\hat{A}_i$ uniformly to each token within the response. 

To correctly utilize the advantage in Eq.(\ref{eq:grpo_advantage}), we need to treat the whole response as an action and compute the importance ratio at the sequence level. The sequence-level importance ratio can be derived from the likelihood decomposition, $\frac{\pi_\theta(o_i|q)}{\pi_{\theta_{old}}(o_i|q)} = \frac{\pi_\theta(o_{i,1}|q)\pi_\theta(o_{i,2}|q,o_{i,<2})\dots}{\pi_{\theta_{old}}(o_{i,1}|q) \pi_{\theta_{old}}(o_{i,2}|q,o_{i,<2})\dots}$. To avoid numerical instability caused by the cumulative multiplication, we can adopt the geometric mean. Accordingly, the importance ratio $w_i(\theta)$ for the response $o_i$ is defined as follows:
\begin{equation}
  w_i(\theta) = \bigg(\frac{\pi_\theta(o_i|q)}{\pi_{\theta_{old}}(o_i|q)}\bigg)^{\frac{1}{|o_i|}} = {\rm exp}\bigg(\frac{1}{|o_i|}\sum_{t=1}^{|o_i|}{\rm log}\frac{\pi_\theta(o_{i,t}|q,o_{i,<t})}{\pi_{\theta_{old}}(o_{i,t}|q,o_{i,<t})}\bigg)
\end{equation}
Then, the calibrated GRPO optimization objective is defined as follows:
\begin{equation}
\label{eq:grpo_c}
    \mathcal{J}_{GRPO-C}(\theta) = \mathbb{E}_{q \sim \mathcal{D}, \{o_i\}_{i=1}^{N} \sim \pi_{\theta_{old}}(\cdot|q)} \bigg[\frac{1}{N}\sum_{i=1}^{N}
     {\rm min}\Big(w_i(\theta)\hat{A}_i,{\rm clip}(w_i(\theta), 1 - \epsilon, 1 + \epsilon)\hat{A}_i\Big) \bigg]
\end{equation}

However, assigning a single advantage to the whole response still suffers from coarse granularity, especially in multi-turn agentic tasks that are inherently step-by-step. To this end, we further analyze how to perform step-level advantage estimation that can adhere to the standard definition. For $s_1$ (the prompt) and $a_{i,1}$ (the first step), the advantage is $\hat{A}_{i,1} = Q(s_1, a_{i,1}) - V(s_1) = r_{i,1} - \frac{1}{N}\sum_{j=1}^{N}r_{j,1}$, where $r_{i,1}$ is the reward for the first step of the $i$-th response. Here, we use the reward to approximate the return, strictly according to the definition, it should be $Q(s_1, a_{i,1}) = r_{i,1} + \gamma r_{i,2} + \dots$. Then, for $s_1$ and $a_1$, the advantage estimation can align with GRPO. However, the situation diverges from the second step onward. Since $s_2$ is the concatenation of $s_1$ and $a_1$, there is no guarantee that $s_{i,2}$ remains identical across all $N$ responses. The worst case is that if every $a_{i,1}$ differs, each $s_{i,2}$ will be distinct. For $s_{i,2}$ and $a_{i,2}$, the value $V(s_{2})$ cannot be reliably estimated because there are not $N$ samples of $a_{i,2}$ conditioned on the same $s_{i,2}$. Despite this, in GRPO with process supervision \citep{GRPO, prime}, the advantage is directly calculated by $\hat{A}_{i,t} = \frac{r_{i,t} - {\rm mean}(r_{1,t}, r_{2,t}, \dots, r_{N(t),t})}{{\rm std}(r_{1,t}, r_{2,t}, \dots, r_{N(t),t})}$, where $N(t)$ is the number of responses that reach step $t$. This introduces a systematic bias: actions $a_{j,t}, j \neq i$ that don't belong to $s_{i,t}$ are also incorrectly used to estimate $V(s_{t})$. To mitigate this bias, a straightforward solution is to adopt the Monte Carlo sampling. For $s_{i,t}$, we treat it as the prefix and sample $N-1$ extra $\{a^j_{i,t}\}_{j=1}^{N-1}$ from the LLM. With the existing $a_{i,t}$, these $N$ actions enable an unbiased estimation of the advantage $\hat{A}_{i,t}$. 




\section{Proposed Method}
\label{sec:method}

\begin{figure*}[ht]
    \centering
    \includegraphics[width=\textwidth]{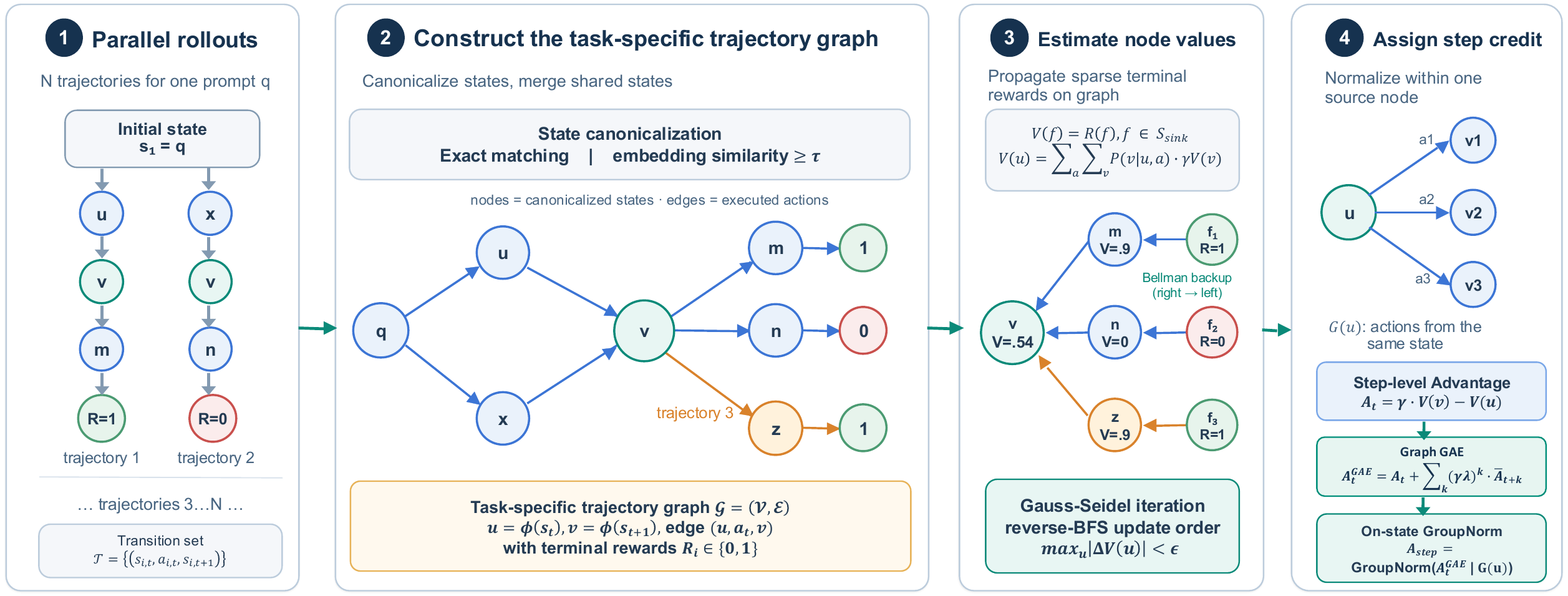}
    \caption{Method overview. Given $N$ rollouts for each prompt, GRAFT merges trajectories into a trajectory graph. Terminal rewards are then propagated backward via Bellman iteration to estimate node state-values. The state-value difference between two nodes \(\gamma V(s_{t+1})-V(s_t)\) is assigned as the step-level advantage, which coincides with the standard advantage function definition (see Sec.\ref{sec:step_advantage}).}
    \label{fig:framework}
\end{figure*}

The analysis in Sec.\ref{sec:preliminary} reveals a fundamental challenge in applying GRPO-style advantage estimation to agentic RL: the value function $V_\pi(s_{i,t})$ at an intermediate state $s_{i,t}$ cannot be reliably estimated unless multiple actions are sampled from the same state. The sampling-based remedy outlined in Sec.\ref{sec:preliminary} is theoretically sound but computationally prohibitive, requiring $\mathcal{O}(N^2 T)$ rollout cost.

To mitigate this issue while preserving definition-adherent estimation, we should aggregate similar states across trajectories to better leverage global information. Crucially, different trajectories are not fully independent, different actions can also reach semantically identical intermediate states (see experiments in Appendix \ref{sec:node_group_size}). This means all rollout data from a prompt can form a directed graph, where nodes represent states and edges represent actions, rather than being regarded as $N$ isolated chains. By aggregating all edges outgoing from the same node, we obtain a natural group of actions sampled from the same state, enabling more faithful step-level advantage estimation at no additional rollout cost. The full framework is illustrated in Fig.\ref{fig:framework}. Note that GiGPO \citep{GiGPO} has employed the way of aggregating cross-trajectory steps sharing the same state into a group to compute step-level advantages within the group. Compared to this plain method of aggregating all steps into multiple independent groups, we further discuss that our method would have more advantages over GiGPO in Appendix \ref{sec:discuss_with_gigpo}. 


\subsection{Graph Construction}

\textbf{Formal definition.} Let $\mathcal{T} = \{(s_{i,t}, a_{i,t}, s_{i,t+1})\}$ denotes the set of all transition tuples collected from $N$ parallel rollouts for a given task with prompt $q$. We construct a task-specific directed graph $\mathcal{G} = (\mathcal{V}, \mathcal{E})$, where each node $u,v \in \mathcal{V}$ corresponds to a canonicalized state, and each edge $(u,v,a) \in \mathcal{E}$ corresponds to a transition $s_t \xrightarrow{a} s_{t+1}$ such that $u = \phi(s_t)$ and $v = \phi(s_{t+1})$.

\textbf{State canonicalization.} In our framework, the state represents both the predefined environment state and the LLM response prefix up to a specific step (as discussed in Sec.\ref{sec:preliminary}). Two states $s$ and $s^\prime$ will be mapped to the same node if $\phi(s) = \phi(s')$. In practice, we can employ two canonicalization modes: (i) \textbf{exact matching}, applied when all states are predefined strings provided by the environment, here $\phi(s)$ is a deterministic hash of the state string; and (ii) \textbf{similarity-based matching}, applied when states are derived from raw LLM contexts, here $s$ and $s'$ are merged into a node if the cosine similarity between their encoded embedding vectors exceeds a threshold $\tau$, and $\phi(s)$ returns a unique identifier (UID) for the resulting cluster. The latter handles a key practical reality: language models frequently introduce minor lexical variations even when generating semantically identical content, rendering exact string matching insufficient for reliable state identification in unstructured settings.

\textbf{Sink nodes.} For each trajectory $i$, the final state $s_{i,T}$ (the state terminated or reached the max turns) is designated as a sink node with terminal reward $R_i \in \{0, 1\}$, where $0$ and $1$ indicate the failure and success of the trajectory, respectively.

\subsection{State-Value Estimation via Bellman Iteration}

Given the graph $\mathcal{G}$ and the rewards in the terminal nodes, our goal is to assign a state-value estimate $V(u)$ to every node $u \in \mathcal{V}$. To leverage the graph's transition structure, we can perform the classical Bellman value iteration on the graph: starting from the terminal nodes with their outcome rewards, we iteratively propagate value estimates backward through the graph until convergence. 

\textbf{Bellman equation.} For sink nodes $f \in \mathcal{S}_{\text{sink}}$, the value is fixed to the terminal reward: $V(f) = R(f), f \in \mathcal{S}_{\text{sink}}$. For non-sink nodes, we adopt the ``action-level'' aggregation scheme, where the node value is computed as the mean over distinct actions of the expected next-state value:
\begin{equation}
\label{eq:bellman_equation}
    V(u) = \sum_{a \in \mathcal{A}(u)}\hat{\pi}(a \mid u) \sum_{v \in \mathcal{V}} \hat{P}(v \mid u, a) \cdot \gamma \cdot V(v)
\end{equation}
where $\mathcal{A}(u)$ is the set of distinct actions executed from node $u$, and $\hat{\pi}(a \mid u)$ and $\hat{P}(v \mid u, a)$ are the empirical transition probability estimated from edge counts:
\begin{equation}
    \hat{\pi}(a \mid u) = \frac{\text{count}(u \xrightarrow{a} \cdot)}{\sum_{a^\prime} \text{count}(u \xrightarrow{a^\prime} \cdot)}, \quad\quad \hat{P}(v \mid u, a) = \frac{\text{count}(u \xrightarrow{a} v)}{\sum_{v'} \text{count}(u \xrightarrow{a} v')}
\end{equation}

\textbf{Theoretical connection.} The action-level Bellman equation is a direct empirical approximation of the standard value function (Eq.(\ref{eq:definition})). Under sparse reward settings (rewards are assigned only at sink nodes), the state-value satisfies $V(s_t) = \mathbb{E}_{a_t}[Q(s_t, a_t)]$ and $Q(s_t, a_t) = \mathbb{E}_{s_{t+1}}[\gamma V(s_{t+1})]$ (see Sec.\ref{sec:preliminary}). Combining these yields:
\begin{equation}
\label{eq:theory_connection}
    V(s_t) = \mathbb{E}_{a_t}\left[\mathbb{E}_{s_{t+1}}[\gamma V(s_{t+1})]\right] = \sum_{a \in \mathcal{A}(s_t)}\hat{\pi}(a \mid u) \sum_{v} \hat{P}(v \mid s_t, a) \cdot \gamma \cdot V(v)
\end{equation}
which is precisely the action-level update rule defined in Eq.(\ref{eq:bellman_equation}).

\textbf{Iterative solver.} We solve the Bellman equation via Gauss-Seidel value iteration \citep{Gauss-Seidel}. To accelerate convergence, we determine the updating order by computing the reverse-BFS distance from each node to the nearest sink node: nodes closer to a sink node are updated earlier, so information propagates from sinks back to source nodes in as few sweeps as possible. Nodes that cannot reach any sink node retain $V(u) = 0$, which constitutes the correct fixed point under sparse rewards. The iteration terminates when $\max_u |V^{(k+1)}(u) - V^{(k)}(u)| < \epsilon$ (default $\epsilon = 10^{-8}$). Since $\gamma < 1$, the Bellman operator is a contraction mapping and convergence is guaranteed. The maximum number of iterations is set adaptively as $k_{\max} = \lceil \log(\epsilon) / \log(\gamma) \rceil + 50$ to ensure convergence for any $\gamma \in (0, 1)$.

\subsection{Step-Level Advantage Estimation}
\label{sec:step_advantage}

Given the converged value estimates $\{V(u)\}_{u \in \mathcal{V}}$, we define the step-level advantage for each executed transition $(s_t, a_t, s_{t+1})$ as:
\begin{equation}
\label{eq:step_advantage}
    A_t = \gamma \cdot V(\phi(s_{t+1})) - V(\phi(s_t))
\end{equation}

Notably, as analyzed in Sec.\ref{sec:preliminary}, we know that $A(s_t,a_t) = Q(s_t, a_t) - V(s_t) = r(s_t) + \gamma V(s_{t+1}) - V(s_t)$. In sparse-reward settings where the environment step reward  $r(s_t)$ is zero, the $A_t$ coincides exactly with the standard advantage function definition (\emph{i.e.}, which is also called TD(0) in RL). Intuitively, $A_t > 0$ indicates that action $a_t$ moves the agent toward a higher-value state (closer to success), and $A_t < 0$ signals a transition toward a lower-value state.

\begin{proposition}[Lower Advantage Estimation Error]
\label{prop:advantage_error} Given rollouts sampled from a fixed policy $\pi$ in a finite-horizon deterministic environment with sparse rewards. Define 
$$\Omega
=
\left\{
(s,a):
0<
\Pr\!\left(R(\tau)>0\mid S_t=s,A_t=a\right)
<1
\right\}$$
as the set of state-action pairs that admit both successful and failed
continuations. Let
\[
X_t^{G}
=
r_t+\gamma V^\pi(S_{t+1})-V^\pi(S_t)
\]
denote the graph-based step-level credit, and let
\[
X_t^{S}
=
G_t-V^\pi(S_t),
\qquad
X_t^{E}
=
R(\tau)-\mathbb E_{\tau \sim \pi}[R(\tau)\mid S_0=s_0]
\]
denote the population counterparts of the step-level credit of GiGPO and the trajectory-level credit of GRPO, respectively.
Then
\[
\mathcal E(X^{G})
\le
\min\left\{
\mathcal E(X^{S}),
\mathcal E(X^{E})
\right\}
\]
where $\mathcal E(X)
=
\mathbb E_{\tau\sim\pi}
\left[
\bigl(X_t-A^\pi(S_t,A_t)\bigr)^2
\right]$ is the population estimation error of a credit signal X. Moreover, the inequality is strict whenever $\Omega$ has positive visitation
probability under $\pi$.
\end{proposition}

The proof of Proposition \ref{prop:advantage_error} is provided in Appendix \ref{sec:appendix_proof}. This proposition identifies the structural benefit of our graph-based Bellman credit assignment. At the population level, graph-based TD credit replaces the trajectory-specific sampled continuation with the policy-averaged value of the successor state, thereby producing a Bellman-aligned credit signal for the current decision. Thus, the superiority of our GRAFT arises from merging trajectories at shared states and aggregating information from their different continuations via Bellman backups into a lower-error, state-specific credit signal.

\begin{wrapfigure}{l}{0.5\columnwidth} 
    \centering
    \vspace{-10pt}
    \includegraphics[width=\linewidth]{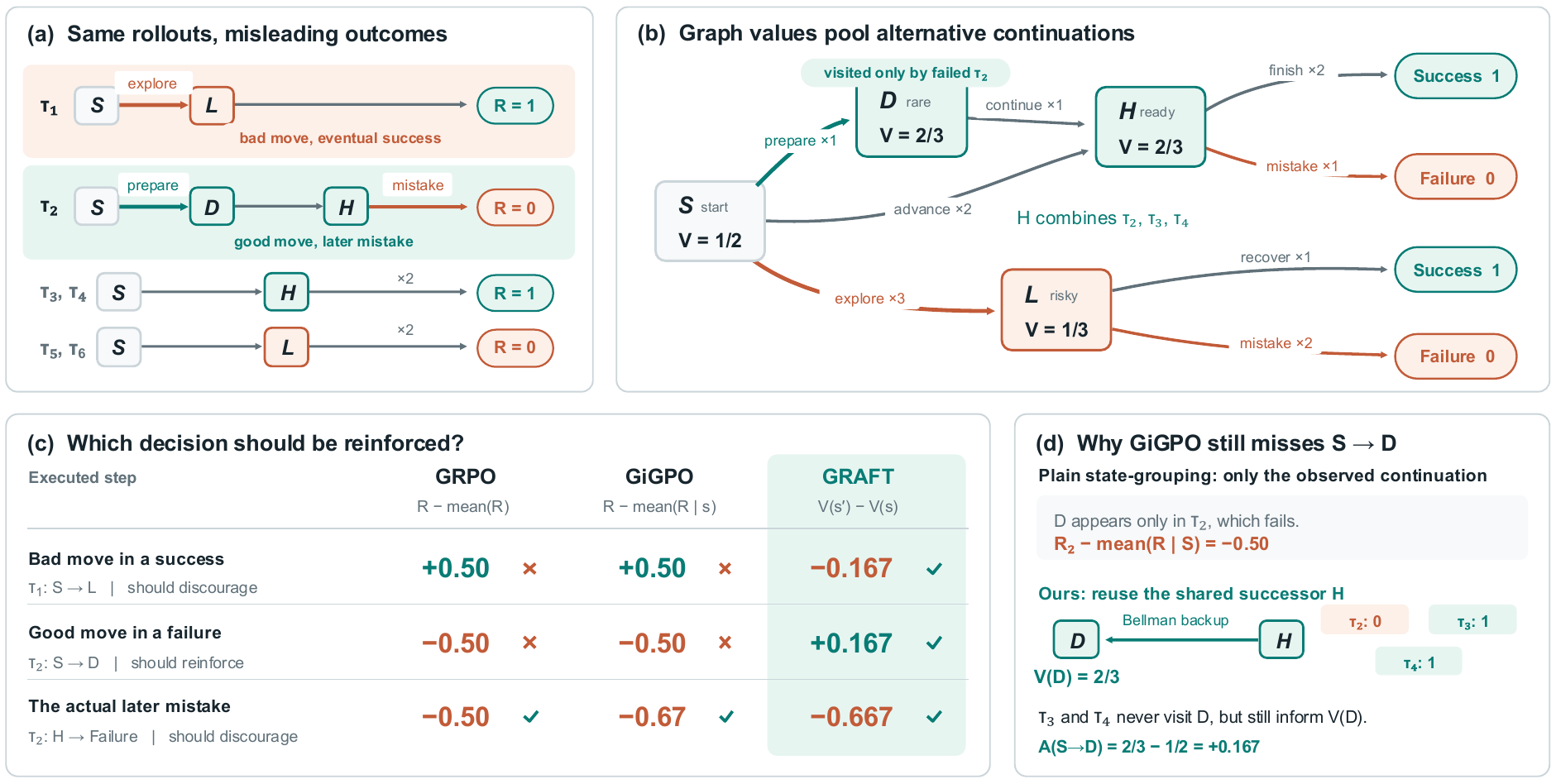}
    \caption{Intuitive comparison of advantage estimation methods: trajectory-level advantage in GRPO, state-grouping-based advantage in GiGPO, and our proposed graph-bootstrapped advantage.}
    \vspace{-5pt}
    \label{fig:intuitive_comparison}
\end{wrapfigure}
To more intuitively illustrate our method's advantages, in Fig.\ref{fig:intuitive_comparison}, we compare our graph-based advantage estimation with the trajectory-level advantage in GRPO and the state-grouping-based approach in GiGPO. By bootstrapping (Bellman backups) over the graph, our method estimates state values more accurately, thus effectively identifying valuable steps within failed trajectories (\emph{e.g.}, $S \rightarrow D$ in $T_2$) and erroneous steps within successful ones (\emph{e.g.}, $S \rightarrow L$ in $T_1$). In contrast, advantage estimation in GRPO and GiGPO is heavily influenced by terminal outcomes (\emph{i.e.}, both $T_1$'s $S \rightarrow L$ and $T_2$'s $S \rightarrow D$ are affected by the outcome of the belonging trajectory, resulting wrong advantages). Thus, they fail to properly credit beneficial actions in failed rollouts or penalize detrimental actions in successful ones. In Appendix \ref{sec:case_study}, we also provide case studies to demonstrate the advantages of our method.

However, since this single-step TD residual $A_t$ relies solely on the state-value difference between two consecutive nodes, any bias in the state-value estimation directly propagates to the advantage estimation. To obtain a more robust advantage estimation, we further take the advantage of Generalized Advantage Estimation (GAE), which is the de facto in RL that can reduce the impact of state-value estimation bias by computing a weighted sum of the subsequent multi-step advantage estimates. Specifically, we propose Graph GAE, which extends GAE to the trajectory graph. For each executed transition $(s_t, a_t, s_{t+1})$, we define the Graph GAE advantage as: 
\begin{equation}
\label{eq:graph_gae}
    A_t^{\text{GAE}} = \underbrace{\gamma V(\phi(s_{t+1})) - V(\phi(s_t))}_{A_t} + \sum_{k=1}^{T-t} (\gamma\lambda)^k \cdot \bar{A}_{t+k}
\end{equation}
where $\bar{A}_{t+k}$ is the mean $k$-hop TD residual, defined as the average single-step TD residual over all raw edge instances $(u', a', v')$ reachable from $a_t$ in exactly $k$ hops on the graph $\mathcal{G}$:
\begin{equation}
    \bar{A}_{t+k} = \frac{1}{|\mathcal{E}_k(a_t)|} \sum_{(u', a', v') \in \mathcal{E}_k(a_t)} \bigl[\gamma V(v') - V(u')\bigr]
\end{equation}
where $\mathcal{E}_k(a_t)$ denotes the set of all directed edges reachable from $a_t$ in exactly $k$ hops. The parameter $\lambda \in [0,1]$ controls the bias-variance trade-off: when $\lambda = 0$, $A_t^{\text{GAE}}$ reduces to the one-step estimate $A_t$ with low variance but high bootstrap bias. As $\lambda$ increases, the effective horizon of the estimator expands, thereby reducing the reliance on intermediate state-value bootstrapping, this mitigates bootstrap bias at the cost of higher variance. Intuitively, if the edges reachable from $a_t$ have overall high advantages, this means that starting from $a_t$ will be easier to reach good states, thereby contributing to increase the advantage estimate for $a_t$; Conversely, if the reachable edges yield predominantly low advantages, the advantage of $a_t$ is correspondingly reduced. The characteristic of Graph GAE is that it smoothly interpolates between two extremes—the single-step advantage and the Monte Carlo advantage $\sum_{k=0}^{T-t}\gamma^kr(s_{t+k}) - V(s_t)$—via $\lambda$, achieving a better bias-variance trade-off.

\textbf{Advantage normalization.} Although $A_t^{GAE}$ itself can serve as a robust step-level advantage estimate, to further eliminate scale sensitivity, we additionally apply advantage normalization within each node's on-state group (please see the experiments in Tab.\ref{tab:ablation_studies}). For each node $u$, all outgoing edges $(u, a, v)$ form a natural step group corresponding to the set of actions sampled from the same state $s = \phi^{-1}(u)$. Within this group, we apply advantage normalization:
\begin{equation}
    \hat{A}_{\text{step}}(s_t, a_t) = \frac{A^{GAE}_t - \mu_{\mathcal{G}(s_t)}}{\sigma_{\mathcal{G}(s_t)} + \epsilon}
\end{equation}
where $\mu_{\mathcal{G}(s_t)}$ and $\sigma_{\mathcal{G}(s_t)}$ are the mean and standard deviation of step advantages within the step group $\mathcal{G}(s_t) = \{A_t^{GAE} : (s_t, \cdot, \cdot) \in \mathcal{E}\}$. For groups containing only one edge, to avoid yielding a normalized advantage of zero, we perform normalization by setting the mean to 0 and using the absolute value as the standard deviation. The whole GRAFT algorithm is in Alg.~\ref{alg:graft} in Appendix \ref{sec:overall_alg}.

\vspace{-5pt}
\section{Experiments}
\subsection{Experimental Setup}

\begin{table*}[t]
\centering
\caption{Performance on ALFWorld and WebShop. Results are averaged over 3 random seeds. For ALFWorld, we report the average success rate (\%) for each subtask as well as the overall result. For WebShop, we report both the average score and the average success rate (\%). Best results are \textbf{bolded}.}
\label{tab:main_results}
\resizebox{\linewidth}{!}{
\begin{tabular}{ll ccccccc cc}
\toprule
\multirow{2}{*}{\textbf{Type}} & \multirow{2}{*}{\textbf{Method}}
  & \multicolumn{7}{c}{\textbf{ALFWorld}}
  & \multicolumn{2}{c}{\textbf{WebShop}} \\
\cmidrule(lr){3-9} \cmidrule(lr){10-11}
 & & \textbf{Pick} & \textbf{Look} & \textbf{Clean} & \textbf{Heat} & \textbf{Cool} & \textbf{Pick2} & \textbf{All}
   & \textbf{Score} & \textbf{Succ.} \\
\midrule
\multicolumn{11}{c}{\textit{Qwen2.5-1.5B-Instruct}} \\
\midrule
Prompting   & Qwen2.5            &  5.9 &  5.5 &  3.3 &  9.7 &  4.2 &  0.0 &  4.1 & 23.1 &  5.2 \\
Prompting   & ReAct              & 17.4 & 20.5 & 15.7 &  6.2 &  7.7 &  2.0 & 12.8 & 40.1 & 11.3 \\
Prompting   & Reflexion          & 35.3 & 22.2 & 21.7 & 13.6 & 19.4 &  3.7 & 21.8 & 55.8 & 21.9 \\
RL Training & PPO  & 64.8$\pm$3.5 & 40.5$\pm$6.9 & 57.1$\pm$4.9 &  60.6$\pm$6.6 & 46.4$\pm$4.0 & 47.4$\pm$1.9 & 54.4$\pm$3.1 & 73.8$\pm$3.0 & 51.5$\pm$2.9 \\

RL Training & RLOO               & 88.3$\pm$3.0 & 52.8$\pm$8.6 & 71.0$\pm$5.9 & 62.8$\pm$8.7 & 66.4$\pm$5.5 & 56.9$\pm$4.7 & 69.7$\pm$2.5 & 73.9$\pm$5.6 & 52.1$\pm$6.7 \\

RL Training & GRPO               & 85.3$\pm$1.5 & 53.7$\pm$8.0 & 84.5$\pm$6.8 & 78.2$\pm$7.9 & 59.7$\pm$5.0 & 53.5$\pm$5.6 & 72.8$\pm$3.6 & 75.8$\pm$3.5 & 56.8$\pm$3.8 \\

RL Training & GiGPO              & 94.4$\pm$5.9 & 67.5$\pm$4.6 & \underline{94.8}$\pm$3.8 & 94.4$\pm$7.8 & 79.8$\pm$4.7 & 76.4$\pm$5.4 & 86.7$\pm$1.7 & 83.1$\pm$1.6 & 65.0$\pm$3.2 \\


RL Training & SALT & \underline{96.2}$\pm$1.7 & 65.2$\pm$10.8 & 93.1$\pm$4.7 & 81.8$\pm$8.3 & 85.0$\pm$6.9 & 77.0$\pm$4.7 & 85.2$\pm$2.5 & 86.9$\pm$0.6 & 74.7$\pm$2.4 \\

RL Training & GraphGPO              & 95.2$\pm$1.6 & \underline{85.7}$\pm$5.8 & \textbf{100.0}$\pm$0.0 & \underline{96.3}$\pm$2.6 & \underline{85.3}$\pm$2.6 & \underline{93.7}$\pm$2.2 & \underline{92.7}$\pm$1.3 & \underline{89.3}$\pm$1.5 & \underline{78.7}$\pm$3.9 \\
\cmidrule(lr){1-11}

\rowcolor{paleblue} RL Training & \textbf{GRAFT (Ours)} & \textbf{99.2}$\pm$1.2 & \textbf{98.3}$\pm$2.4 & \textbf{100.0}$\pm$0.0 & \textbf{100.0}$\pm$0.0 & \textbf{93.3}$\pm$3.7 & \textbf{97.6}$\pm$3.3 & \textbf{97.4}$\pm$0.4 & \textbf{90.8}$\pm$0.4 & \textbf{82.3}$\pm$1.0 \\
\midrule
\multicolumn{11}{c}{\textit{Qwen2.5-7B-Instruct}} \\
\midrule
Prompting   & Qwen2.5            & 33.4 & 21.6 & 19.3 &  6.9 &  2.8 &  3.2 & 14.8 & 26.4 &  7.8 \\
Prompting   & ReAct              & 48.5 & 35.4 & 34.3 & 13.2 & 18.2 & 17.6 & 31.2 & 46.2 & 19.5 \\
Prompting   & Reflexion          & 62.0 & 41.6 & 44.9 & 30.9 & 36.3 & 23.8 & 42.7 & 58.1 & 28.8 \\
RL Training & PPO  & 92.3$\pm$4.0 & 64.0$\pm$8.4 & 92.5$\pm$2.4 & 89.5$\pm$7.0 & 80.3$\pm$2.0 & 68.8$\pm$8.3 & 80.4$\pm$2.7 & 81.4$\pm$3.1 & 68.7$\pm$5.1 \\
RL Training & RLOO               & 87.6$\pm$4.3 & 78.2$\pm$8.3 & 87.3$\pm$5.8 & 81.3$\pm$7.6 & 71.9$\pm$5.2 & 48.9$\pm$8.4 & 75.5$\pm$4.6 & 80.3$\pm$3.2 & 65.7$\pm$4.0 \\

RL Training & GRPO               & 90.8$\pm$5.1 & 66.1$\pm$6.7 & 89.3$\pm$5.4 & 74.7$\pm$6.9 & 72.5$\pm$5.4 & 64.7$\pm$7.3 & 77.6$\pm$5.2 & 79.3$\pm$2.8 & 66.1$\pm$3.7 \\

RL Training & GiGPO              & 97.7$\pm$1.6 & 82.7$\pm$7.9 & 98.8$\pm$1.6 & 83.7$\pm$7.2 & 89.3$\pm$8.2 & 79.2$\pm$6.6 & 90.8$\pm$1.3 & 84.4$\pm$2.9 & 72.8$\pm$3.2 \\


RL Training & SALT & 93.4$\pm$2.1 & 72.6$\pm$7.5 & 91.5$\pm$3.2 & 90.3$\pm$3.5 & 78.1$\pm$2.7 & 76.4$\pm$4.4 & 87.3$\pm$4.4 & 83.1$\pm$3.8 & 75.2$\pm$5.5 \\

RL Training & GraphGPO              & \textbf{100.0}$\pm$0.0 & \underline{92.9}$\pm$5.8 & \textbf{100.0}$\pm$0.0 & \underline{94.4}$\pm$0.0 & \underline{91.4}$\pm$1.5 & \underline{92.1}$\pm$2.2 & \underline{95.3}$\pm$1.1 & \underline{86.9}$\pm$0.7 & \underline{80.3}$\pm$1.3 \\
\cmidrule(lr){1-11}

\rowcolor{paleblue} RL Training & \textbf{GRAFT (Ours)} & \underline{99.0}$\pm$1.5 & \textbf{100.0}$\pm$0.0 & \underline{98.4}$\pm$2.3 & \textbf{100.0}$\pm$0.0 & \textbf{97.2}$\pm$2.9 & \textbf{96.8}$\pm$2.6 & \textbf{98.4}$\pm$0.6 & \textbf{91.5}$\pm$1.3 & \textbf{82.8}$\pm$0.6 \\
\bottomrule
\end{tabular}}
\end{table*}

\textbf{Benchmarks.} Following the agentic RL baseline GiGPO \citep{GiGPO}, we evaluate our method on a suite of challenging multi-turn agentic benchmarks, including ALFWorld \citep{Alfworld}, WebShop \citep{Webshop}, and SearchQA \citep{SearchQA}. The introduction of these benchmarks is deferred to Appendix \ref{sec:experiment_details}. In SearchQA, the training set is drawn from NQ and HotpotQA, making these two datasets for in-domain evaluation, while the remaining datasets are used to assess out-of-domain generalization. 

\textbf{Baselines.} To ensure a standardized and rigorous comparison, we adopt the baseline results directly from the original GiGPO \citep{GiGPO} and GraphGPO \citep{GraphGPO} papers. For ALFWorld and WebShop, the baselines include: (1) Prompting-based agents: ReAct \citep{react} and Reflexion \citep{reflexion}; and (2) RL training methods: PPO \citep{PPO}, a widely used critic-based RL algorithm, group-based RL methods including RLOO \citep{rloo}, GRPO \citep{GRPO}, and GiGPO \citep{GiGPO}, as well as graph-based policy optimization methods, SALT \citep{SALT}, GraphGPO \citep{GraphGPO} (In appendix \ref{sec:discuss_with_graphgpo}, \ref{sec:discuss_with_salt}, we provide detailed discussions about the differences between our method and these two methods). For searchQA tasks, we also follow the experimental protocol in GiGPO \citep{GiGPO} and compare GRAFT against a specific suite of baselines including R1-Instruct, Search-R1 \citep{SearchQA}, ZeroSearch \citep{zerosearch}, StepSearch \citep{stepsearch} and the agentic RL methods: GiGPO, GraphGPO. 

\textbf{Implementation Details.} To ensure a direct and fair comparison with the prior methods, we utilize Qwen2.5-Instruct series (1.5B, 3B, 7B) as our base models (the compared methods report results based on Qwen2.5-Instruct series). For all benchmarks, we follow GiGPO's training and evaluation configurations as our basic configurations, including rollout batch size, group size, mini-batch size, learning rate, and sampling temperature, etc. Additionally, the specific hyperparameters of our method are: the value discount factor $\gamma = 0.99$, the advantage weighting factor $\lambda = 0.95$ for ALFWorld, $0.8$ for WebShop, and $0.6$ for SearchQA. Further implementation details are provided in Appendix \ref{sec:experiment_details}.

\subsection{Main Results}

As shown in Tab.\ref{tab:main_results}, GRAFT achieves significant gains over the trajectory-level baseline GRPO and demonstrates superior performance compared to the state-of-the-art GiGPO and GraphGPO across both ALFWorld and WebShop. On ALFWorld, GRAFT achieves improvements on nearly all subtasks, resulting in average success rate gains of 24.6\% and 20.8\% over GRPO for the 1.5B and 7B models, respectively. On WebShop, GRAFT not only attains higher task scores but also improves the average success rate over GRPO by 25.5\% and 16.7\% for the 1.5B and 7B models, respectively. These results highlight that our method effectively overcomes the limitations of coarse-grained trajectory-level advantages. In addition, GRAFT also surpasses both step-group-based GiGPO and graph-based GraphGPO by 10.7\%/17.3\% and 4.7\%/3.6\% on ALFWorld and WebShop, respectively, both of which are step-level credit assignment algorithms. As the advantage estimation in our method closely revolves around the standard definition, compared to heuristic advantage estimation approaches in GiGPO, SALT, and GraphGPO, our method can provide more credible step-level advantage estimates.

Tab.\ref{tab:search_results} presents the results on searchQA tasks. We observe that GRAFT achieves strong and consistent gains across both single-hop and multi-hop reasoning datasets. Notably, GRAFT reaches an average success rate of 48.6\% at 7B, outperforming prior baselines such as Search-R1 and StepSearch, and also outperforms the strong step-level credit assignment methods, GiGPO and GraphGPO. In Appendix \ref{sec:additional_results}, to demonstrate the generalization of our method, we further provide results on the Qwen3 model series.

\begin{table*}[ht]
\centering
\caption{Performance on searchQA tasks. $\dagger$ and $\star$ indicate in-domain and out-of-domain datasets, respectively. \textbf{Bold} indicates the best performance in each category.}
\label{tab:search_results}
\resizebox{\linewidth}{!}{
\begin{tabular}{ll ccc cccc c}
\toprule
\multirow{2}{*}{\textbf{Type}} & \multirow{2}{*}{\textbf{Method}}
  & \multicolumn{3}{c}{\textbf{Single-Hop QA}}
  & \multicolumn{4}{c}{\textbf{Multi-Hop QA}}
  & \multirow{2}{*}{\textbf{Avg.}} \\
\cmidrule(lr){3-5} \cmidrule(lr){6-9}
 & & \textbf{NQ}$\dagger$ & \textbf{TriviaQA}$\star$ & \textbf{PopQA}$\star$
   & \textbf{HotpotQA}$\dagger$ & \textbf{2Wiki}$\star$ & \textbf{MuSiQue}$\star$ & \textbf{Bamboogle}$\star$ & \\
\midrule
\multicolumn{10}{c}{\textit{Qwen2.5-3B-Instruct}} \\
\midrule
RL Training & R1-Instruct          & 27.0          & 53.7          & 19.9          & 23.7          & 29.2          &  7.2          & 29.3          & 27.1 \\
RL Training & Search-R1            & 34.1          & 54.5          & 37.8          & 32.4          & 31.9          & 10.3          & 26.4          & 32.5 \\
RL Training & ZeroSearch           & 41.4          & 57.4          & 44.8 & 27.4       & 30.0          &  9.8          & 11.1          & 31.7 \\
RL Training & StepSearch           & --            & --            & --            & 34.5          & 32.0          & \textbf{17.4} & --            & 34.4 \\
RL Training & GiGPO                & 42.0 & 59.5 & 42.4      & \underline{36.9} & 37.0 & \underline{12.6}       & \underline{64.1} & 42.1 \\


RL Training & GraphGPO                & \underline{44.5} & \underline{59.7} & \underline{46.2}      & 36.2 & \underline{37.2} & 12.1       & 63.7 & \underline{44.0} \\
\cmidrule(lr){1-10}

\rowcolor{paleblue} RL Training & \textbf{GRAFT (Ours)} & \textbf{45.3} & \textbf{61.8} & \textbf{46.4} & \textbf{37.2} &  \textbf{37.3}         & 12.4 & \textbf{64.9} & \textbf{45.2} \\

\midrule

\multicolumn{10}{c}{\textit{Qwen2.5-7B-Instruct}} \\
\midrule
RL Training & R1-Instruct          & 21.0          & 44.9          & 17.1          & 20.8          & 27.5          &  6.0          & 19.2          & 22.4 \\
RL Training & Search-R1            & 39.3          & 61.0          & 39.7          & 37.0          & 40.1          & 14.6          & 36.8          & 38.5 \\
RL Training & ZeroSearch           & 43.6          & 61.8          & \textbf{51.5} & 34.6          & 35.2          & 18.4          & 27.8          & 39.1 \\
RL Training & StepSearch           & --            & --            & --            & 38.6          & 36.6          & \textbf{22.6} & --            & 40.0 \\
RL Training & GiGPO                & 46.4 & 64.7 & 46.1       & \underline{41.6} & \textbf{43.6} & \underline{18.9}       & 68.9 & 47.2 \\


 RL Training & GraphGPO                & \underline{46.8} & \underline{65.4} & 47.7       & 41.1 & 42.7 & 17.0       & \underline{69.0} & \underline{48.0} \\
\cmidrule(lr){1-10}

\rowcolor{paleblue} RL Training & \textbf{GRAFT (Ours)} & \textbf{46.9} & \textbf{65.6} & \underline{47.8} & \underline{41.9} & \underline{43.4} &  17.7   & \textbf{70.2} & \textbf{48.6} \\
\bottomrule
\end{tabular}}
\end{table*}

\subsection{Ablation Studies \& Further Analysis}

Tab.~\ref{tab:ablation_studies} isolates the effects of the key components of our GRAFT. We start with a baseline GRAFT$^\dagger$ that uses the step-level advantage defined in Eq.(\ref{eq:step_advantage}). Removing advantage normalization suffers a substantial performance drop across both benchmarks (\emph{e.g.}, ALFWorld's ``All'' declines from 96.10 to 71.33, and WebShop success rate falls to 72.93), demonstrating that normalization is critical for stabilizing gradient updates and preventing scale-induced optimization instability. When we replace the original GRPO objective with the GRPO-C objective (Eq.(\ref{eq:grpo_c})), performance surpasses the baseline across nearly all metrics. This confirms that aligning the optimization granularity with the advantage estimator (treating the entire step as a single action to match the step-level advantage) is superior to GRPO's mismatched token-level assignment. Integrating Graph GAE further delivers consistent gains, analogous to the bias-variance trade-off in standard GAE, our Graph GAE effectively fuses multiple k-step estimators to produce a more faithful and robust advantage estimate than the single-step estimator, leading to improved policy optimization. The complete GRAFT framework achieves the highest performance, confirming that these components are not independent but can synergistically enhance policy learning in long-horizon agentic tasks.

\begin{table*}[ht]
\centering
\caption{Ablation study results. ``w/o adv normalization'' denotes directly using the step-level advantage defined in Eq.(\ref{eq:step_advantage}) without further advantage normalization. GRAFT$^\dagger$ is a variant of GRAFT without GRPO-C objective (Eq.(\ref{eq:grpo_c})) and Graph GAE, serving as the baseline for ablation studies. The results are based on Qwen2.5-1.5B-Instruct.}
\label{tab:ablation_studies}
\resizebox{\linewidth}{!}{
\begin{tabular}{l | ccccccc | cc}
\toprule
\multirow{2}{*}{\textbf{Method}}
  & \multicolumn{7}{c|}{\textbf{ALFWorld}}
  & \multicolumn{2}{c}{\textbf{Webshop}} \\
 & Pick & Look & Clean & Heat & Cool & Pick2 & All & Score & Succ. \\
\midrule
GRAFT$^\dagger$
  & 99.23$\pm$1.08
  & 100.0$\pm$0.00
  & 96.73$\pm$2.31
  & 97.23$\pm$3.15
  & 92.57$\pm$5.35
  & 89.60$\pm$2.46
  & 96.10$\pm$1.13
  & 89.53$\pm$0.76
  & 80.67$\pm$1.68\\
w/o adv normalization
  & 75.37$\pm$1.51
  & 75.47$\pm$3.53
  & 72.50$\pm$4.76
  & 63.10$\pm$3.53
  & 72.20$\pm$5.67
  & 60.90$\pm$4.54
  & 71.33$\pm$2.87
  & 85.67$\pm$0.74
  & 72.93$\pm$1.46\\
\midrule
w/ GRPO-C objective (\ref{eq:grpo_c})
  & 100.0$\pm$0.00
  & 98.03$\pm$2.78
  & 96.00$\pm$3.39
  & 100.0$\pm$0.00
  & 96.20$\pm$2.53
  & 90.73$\pm$3.78
  & 96.87$\pm$1.08
  & 90.63$\pm$1.51
  & 81.30$\pm$1.31\\
w/ Graph GAE (\ref{eq:graph_gae})
  & 97.50$\pm$2.18
  & 100.0$\pm$0.00
  & 100.0$\pm$0.00
  & 100.0$\pm$0.00
  & 97.10$\pm$2.11
  & 91.53$\pm$4.13
  & 97.17$\pm$0.75
  & 90.23$\pm$0.70
  & 81.73$\pm$0.75\\
GRAFT
  & 99.17$\pm$1.17
  & 98.33$\pm$2.35
  & 100.0$\pm$0.00
  & 100.0$\pm$0.00
  & 93.33$\pm$3.71
  & 97.63$\pm$3.34
  & 97.43$\pm$0.38
  & 90.83$\pm$0.37
  & 82.27$\pm$0.99\\
\bottomrule
\end{tabular}}
\end{table*}


\begin{wrapfigure}{r}{0.4\columnwidth} 
    \centering
    \vspace{-10pt}
    \includegraphics[width=\linewidth]{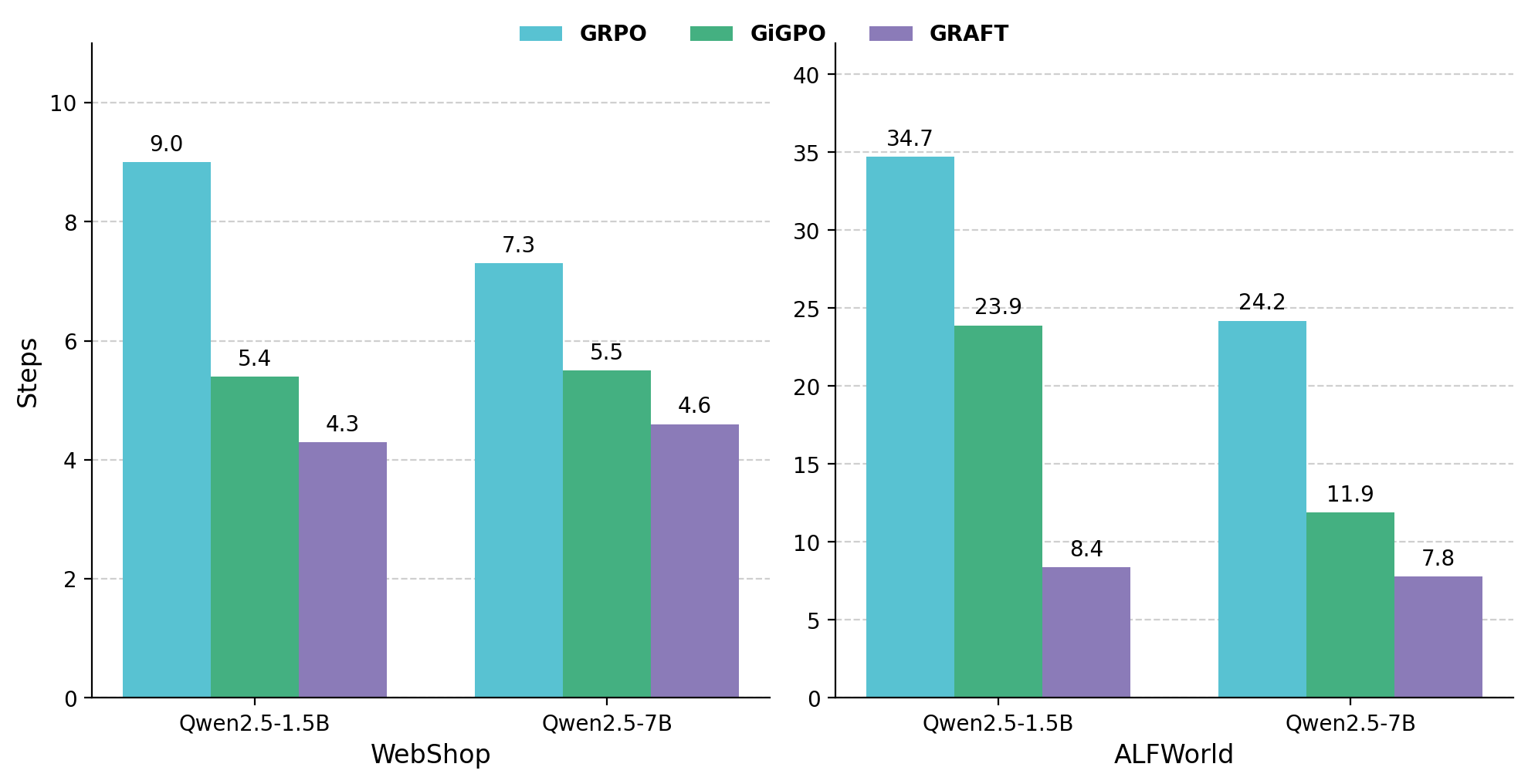}
    \caption{The average number of execution turns required to finish tasks.}
    \vspace{-8pt}
    \label{fig:inference-step}
\end{wrapfigure}
\textbf{Task Efficiency Analysis}. In Fig.\ref{fig:inference-step}, we further illustrate the average number of interaction turns required by the trained models to complete tasks on WebShop and ALFWorld. Across both benchmarks, GRAFT-trained models consistently require fewer turns than those trained with GRPO or GiGPO. This indicates that our method can yield more faithful step-level advantage estimates, which enable the agent to avoid redundant actions and make more efficient decisions. Fewer interaction turns inherently translate to lower model inference costs and reduced environment API calls. Together with Tab.\ref{tab:main_results}, these results demonstrate that our GRAFT achieves substantial advantages over existing methods in both task accuracy and cost efficiency.

In appendix \ref{sec:additional_experiments}, we further provide additional experimental analysis, including training dynamics, hyperparameter experiments, method mechanism analysis, additional results, and case studies.

\begin{wrapfigure}{r}{0.4\columnwidth} 
    \centering
    \vspace{-10pt}
    \includegraphics[width=\linewidth]{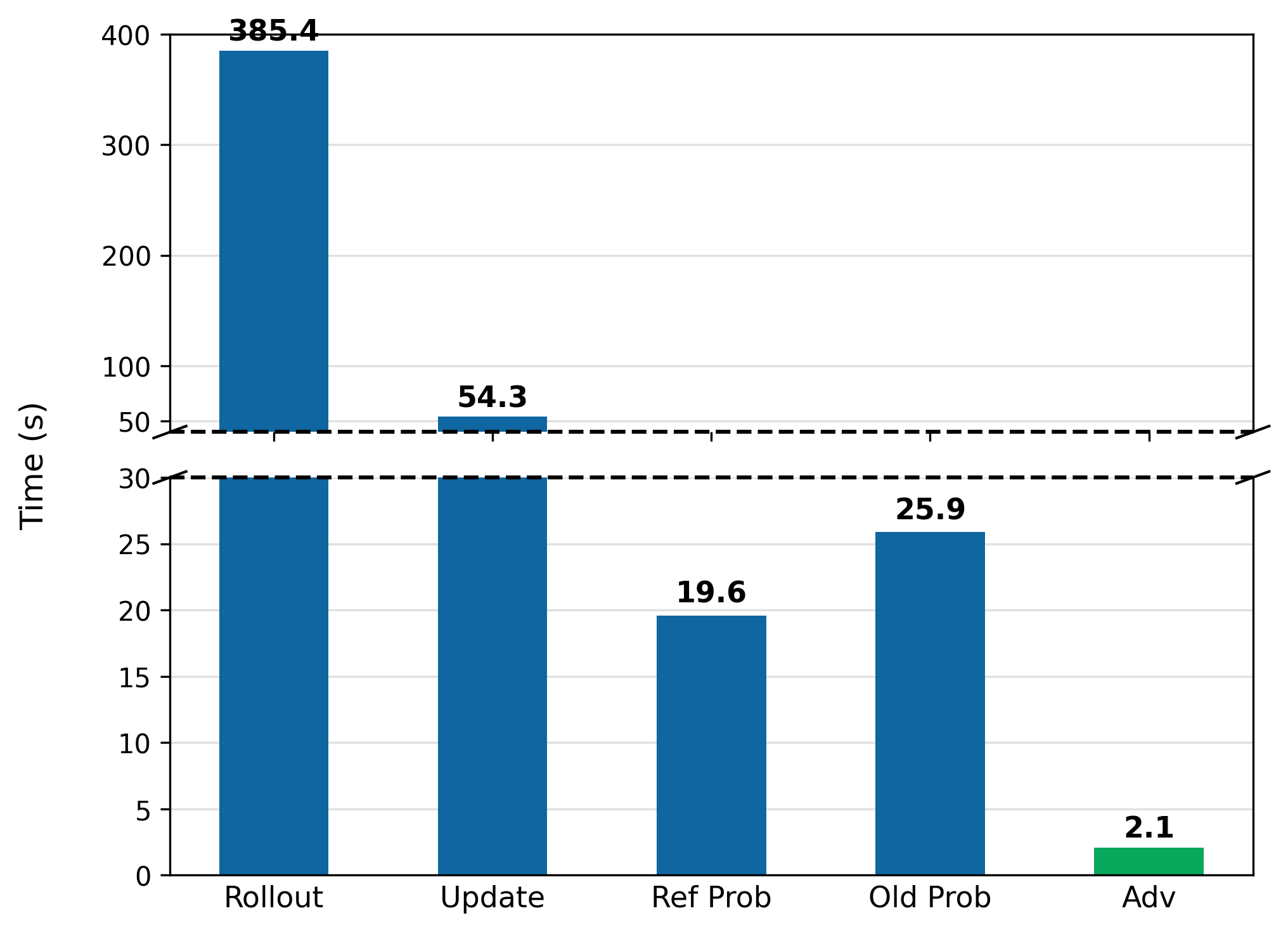}
    \caption{Breakdown of time consumption per training step.}
    \vspace{-5pt}
    \label{fig:time-breakdown}
\end{wrapfigure}
\textbf{Computional Cost.} Our method incurs no additional GPU memory overhead, as we have not introduced any extra models or modules. Due to inconvenience in comparing FLOPs between advantage estimation and model computation, we analyze the time consumption per training step. We use Qwen2.5-1.5B-Instruct as the base model and train it on the ALFWorld benchmark. The per-step runtime breakdown is shown in Fig.\ref{fig:time-breakdown}. As illustrated, a typical training step involves a rollout stage, advantage estimation, probability computation for sampled responses under both the reference and old policies, and policy update. Our method operates specifically within the advantage estimation stage. The additional costs introduced by GRAFT stem from graph construction and graph-based advantage estimation. These costs are negligible compared to rollout and policy update, occupying merely 0.43\% of the total RL training time and introducing virtually no extra computational burden.


\section{Related Work}
\label{sec:related_work}

\textbf{Reinforcement Learning for LLMs}. RL for LLMs has evolved from preference alignment to reasoning enhancement. Early PPO-based RLHF methods \citep{PPO, rlhf, InstructGPT}, while effective, were computationally burdensome due to the need for separate reward and value models, prompting a shift toward alternatives like DPO \citep{DPO} and its variants (e.g., SimPO \citep{SimPO}, KTO \citep{KTO}, ORPO \citep{ORPO}). However, the offline characteristics of the DPO series result in inherently performance-bounded compared to online methods. The recent success of DeepSeek-R1 \citep{DeepSeek-R1} established RL with Verifiable Rewards (RLVR) as a new paradigm, demonstrating that strong reasoning can emerge via online outcome-based algorithms (\emph{e.g.}, GRPO \citep{GRPO}, DAPO \citep{DAPO}, GSPO \citep{GSPO}) without SFT cold-start or process supervision. Although such group-based methods enable scalable, memory-efficient training by eliminating critic models, they predominantly rely on sparse outcome rewards and suffer from the credit assignment problem at the step level. While Process Reward Models \citep{prime, prm2, agentprm} offer finer-grained guidance, they incur prohibitive annotation costs and are prone to reward hacking. Consequently, recent efforts \citep{GiGPO, HCAPO, GraphGPO} have pivoted on deriving dense, multi-granularity advantages directly from sampling statistics or state transitions, aiming to achieve precise credit assignment without the overhead of learned critics or expensive process annotations.

\textbf{Agentic Reinforcement Learning}. RL has become a cornerstone for empowering LLM agents in dynamic, open-ended environments \citep{agenticrl, stepsearch, Tongyi-DeepSearch}, evolving from early value-based methods \citep{AWR, value-based-rl} to modern policy gradient approaches for complex tasks like web navigation \citep{web-navgation}, tool use \citep{tool-use}, and software engineering \citep{swe}. Recent works such as Search-R1 \citep{SearchQA} and WebSailor \citep{websailor} demonstrate that agentic RL can effectively instill multi-step information-seeking and long-horizon planning capabilities without human-curated supervision. Despite these advances, training efficacy remains limited by sparse outcome-based rewards, which hinder the agent’s ability to correct intricate intermediate errors. To address this credit assignment problem, Process Reward Models offer step-level supervision but incur prohibitive annotation costs. Alternatively, intrinsic reward mechanisms such as EMPG \citep{EMPG} utilize dynamic entropy to encourage exploration. More recently, GiGPO \citep{GiGPO} introduces state-based anchoring to aggregate all states into step-wise groups for normalized advantage estimation in the aggregated groups. HCAPO \citep{HCAPO} proposes to leverage the LLM itself as a post-hoc critic to refine step-level credit through hindsight reasoning. A recent work, GraphGPO \citep{GraphGPO}, employs the graph structure to model the state-action transition relationship and achieves step-level credit assignment via heuristic reward shaping. Although both use the graph structure, our method differs fundamentally from GraphGPO in motivation, implementation, and effectiveness; we provide a detailed comparative discussion in Appendix \ref{sec:discuss_with_graphgpo}.

\section{Conclusion}

In this work, we revisit the basic advantage estimation problem and group-based RL, and identify a critical gap: GRPO-style advantage estimation is reliable at the response level, but systematically biased when extending the group-based estimation strategy to the step-level. To address this, we propose GRAFT, a graph-based credit-assignment framework that grafts all rollout trajectories into a unified trajectory graph, recovers node state-values via Bellman iteration on the graph, and assigns credit to each edge by the node value difference. The resulting step-level advantage is not a heuristic surrogate but faithfully adheres to the basic advantage definition in RL, and asymptotically converges to the true target as the graph reaches sufficiently saturated. Across diverse multi-turn agentic benchmarks, GRAFT consistently outperforms GRPO and demonstrates superior performance compared to recent SOTA agentic RL algorithms, while remaining critic-free and PRM-free.

\section*{Acknowledgments}

This work was completed during the internship at Tencent AI Platform Department. We gratefully acknowledge Tencent Inc. for providing computing resources that greatly supported the completion of this work. This work was also supported in part by the National Natural Science Fund of China (No.62371295), the Shanghai Jiao Tong University AI for Engineering Initiative (No.WH410263001/001), and the Science and Technology Commission of Shanghai Municipality
(No.22DZ2229005).

\bibliography{iclr2027_conference}
\bibliographystyle{iclr2027_conference}

\clearpage
\appendix
\section*{Appendix}

\section{More Discussions}

\subsection{Advantages over GiGPO}
\label{sec:discuss_with_gigpo}

As stated in Sec.\ref{sec:method}, GiGPO \citep{GiGPO} has proposed to estimate step-level advantages by grouping steps that share the same state and normalizing their discounted Monte Carlo returns $R_t = \sum_{l \geq 0} \gamma^l r_{t+l}$ within each group. Since our method relies on aggregating cross-trajectory steps into graph nodes, we discuss the advantages of our method over the plain state-grouped advantage estimation approach in GiGPO. 

The design in GiGPO has a fundamental limitation: the value signal for state $s$ is derived only from trajectories that pass through $s$, making the estimate high-variance and sample-inefficient, especially for states visited infrequently. Our method addresses this by constructing an empirical transition graph $\mathcal{G}$ over the rollout group and estimating $V(s)$ via Bellman value iteration (see Sec.\ref{sec:method}). The step advantage for each executed edge $(s \xrightarrow{a} s')$ is then defined as $A(s, a, s') = \gamma\, V(s') - V(s)$, which exactly aligns with the standard advantage function in RL. This formulation offers two key advantages over GiGPO's vanilla state-grouping approach:

1. Broader information aggregation. Because $V(s)$ is computed by backward propagation through the graph, it implicitly aggregates reward signals from all trajectories reachable from the successors of $s$, rather than only those passing through $s$ itself. Concretely, if $s$ branches into $s'_1$ and $s'_2$, and these successors are collectively visited by $K$ downstream trajectories, then $V(s)$ incorporates rewards from all $K$ trajectories, whereas GiGPO's all $R_t$ for $s$ are confined to the trajectories that directly visit $s$. In Appendix.\ref{sec:information_degree}, we compare the information propagation degree (the number of distinct trajectories contributing reward information to the value estimation of a node (or state)) between our method and GiGPO. The results show that our method consistently achieves a higher information propagation degree than GiGPO throughout training.

 2. Theoretical soundness. The step advantage $\gamma V(s') - V(s)$ is precisely aligned with the standard definition of the advantage function $A(s, a) = Q(s, a) - V(s)$, since under a deterministic transition $Q(s,a) = \gamma V(s')$. This theoretical grounding ensures that the step-level signal faithfully measures how much better action $a$ is relative to the expected value of state $s$, whereas GiGPO's Monte Carlo return $R_t$ introduces high variance from stochastic future trajectories.

 3. GiGPO can be viewed as a special case of our method in which each trajectory is treated as an independent graph, and Bellman iteration is performed separately on these isolated graphs to estimate node values. Subsequently, step-level groups are constructed across trajectories, and values are normalized within each group to obtain step-level advantages. Compared to graph-based value estimation, single-trajectory value estimation is inherently less accurate due to the stochasticity of individual trajectories, resulting in higher variance.

Together, these properties make our graph-based estimator both more sample-efficient and more theoretically grounded than GiGPO's state-grouping baseline. As for performance, the results in Tab.\ref{tab:main_results} and \ref{tab:search_results} show that our method can achieve significantly better results than GiGPO, validating the superiority of our method over GiGPO in estimating step-level advantages.

\subsection{Discussion with GraphGPO}
\label{sec:discuss_with_graphgpo}

In GraphGPO \citep{GraphGPO}, the authors also utilized the graph structure to model state-action transitions across trajectories, which coincides with the idea in our work. Therefore, it is necessary to clarify the distinctions between our work and GraphGPO. Although both methods use the graph structure, our method differs fundamentally from GraphGPO in motivation, implementation, and performance. We elaborate on these differences below:

1. In terms of motivation, GraphGPO is driven by empirical observations: it finds that 65\% of steps in successful trajectories do not meaningfully advance the task, while 22\% of steps in failed trajectories contribute to task progress. Arguing that trajectory-level credit attribution is too coarse, GraphGPO aggregates all rollout trajectories into a state transition graph to leverage global connectivity for finer-grained credit assignment. In contrast, our motivation stems from the preliminary analysis in Sec.\ref{sec:preliminary}, which demonstrates that extending GRPO-style advantage estimation to the step level introduces systematic bias. To obtain theoretically grounded estimates while not incurring prohibitive Monte Carlo sampling costs, we should aggregate similar states across trajectories to better leverage global information. To this end, the graph structure serves as a natural representation for modeling the cross-trajectory state-action transition relationship.

2. In terms of implementation, GraphGPO adopts a heuristic approach that defines the step-level rewards based on the shortest-path distance from each node to any successful nodes $r(s,a,s^\prime) = r_{succ}w^{d(s)}$, where $d(s)$ is the shortest distance from $s$ to any successful states and $w \in (0, 1)$ is a distance discount factor. Then, the rewards are normalized within each node group to obtain step-level advantages. This heuristic reward assignment approach is based on the assumption that beneficial actions can reduce the distance to the task goal. In contrast, our method is grounded in fundamental RL theory. We treat the constructed graph as an empirical Markov Decision Process (MDP) and apply Bellman iteration ($V(u) = \sum_{a \in \mathcal{A}(u)}\hat{\pi}(a \mid u) \sum_{v \in \mathcal{V}} \hat{P}(v \mid u, a) \cdot \gamma \cdot V(v)$) to estimate the state-value of each node. The step-level advantage between two nodes is then defined as $A_t = \gamma \cdot V(u) - V(v)$. Theoretically, the estimated node values conform to the standard definition of state-value in RL (see Eq.(\ref{eq:theory_connection})), and the derived step-level advantages coincide exactly with the standard advantage function (see Sec.\ref{sec:step_advantage}). Rather than relying on hand-crafted heuristics, our approach is principled and theory-driven. 

3. In GraphGPO, after obtaining the step-level advantages, these values are uniformly allocated to individual tokens within each step and the GRPO optimization objective is applied. However, as our analysis in Sec.\ref{sec:preliminary} demonstrates, this token-level allocation introduces systematic bias. To this end, we employ the calibrated GRPO optimization objective defined in Eq.(\ref{eq:grpo_c}) in our method, which can further outperform the original objective formulation as shown in Tab.\ref{tab:ablation_studies}.

4. As for method performance, Tab.\ref{tab:main_results} and \ref{tab:search_results} both demonstrate that our method consistently outperforms GraphGPO across multiple multi-turn agentic benchmarks, even though both methods are based on graph representation. This indicates that our advantage estimation method, grounded in standard RL definitions, yields more accurate step-level advantages than GraphGPO’s heuristic approach. Furthermore, we provide theoretical guarantees showing that as the graph becomes saturated, the estimated step-level advantages in our method converge to the true advantages.

\subsection{Discussion with SALT}
\label{sec:discuss_with_salt}

In SALT \citep{SALT}, the authors also claimed to propose a step-level advantage assignment method based on the trajectory graph. However, their approach underutilizes the graph structure, which actually can be reduced to the plain state aggregation method akin to GiGPO. Specifically, SALT first computes trajectory-level advantages via GRPO, then identifies identical transitions $(s,a,s')$ across trajectories into a group and averages their inherited advantages (\emph{i.e.}, the trajectory-level advantages) within the group. This mechanism mirrors the state-based grouping in GiGPO, but differently SALT groups actions (along with their adjacent states) and assigns the group-averaged advantage to each action. Consequently, SALT inherits GiGPO’s fundamental limitations as discussed in Appendix \ref{sec:discuss_with_gigpo}: information propagation also remains limited. That is, advantage averaging within each group solely utilizes trajectory-level signals from the corresponding trajectories containing the specific action. Although this heuristic advantage adjustment method can mitigate trajectory-level noise (\emph{e.g.}, if a good action receives a misleading trajectory-level advantage due to subsequent failures, averaging across multiple trajectories may mitigate the credit misalignment issue), it cannot adjust advantages for actions that appear only once across all trajectories. In contrast, our method applies Bellman iteration to estimate the state-value of each node on the graph, deriving step-level advantages from the value differences between two nodes. This formulation does not depend on the occurrences number of the action, and even for actions appearing only once, our method can take advantage of bootstrapping on the graph to achieve a good step-level advantage estimation.

As for performance, the results in
Tab.\ref{tab:main_results} show that our method can achieve significantly better results than SALT, validating the
superiority of our method over SALT in estimating step-level advantages.

\subsection{Discussion with RTMC}
\label{sec:discuss_with_rtmc}

We further discuss with RTMC \citep{RTMC}, the authors proposed Rollout-Tree Monte Carlo advantage estimation, which aggregates return statistics from group rollouts over a tree structure to compute per-step Q-values and advantages, enabling fine-grained credit assignment. Specifically, for each step $t$, the discounted return is computed as $R_t = \sum_{l=0}^{T-t}\gamma^lr_{t+l}$. Then, for each state-action pair $(s, a)$, the Q-value is estimated by averaging returns across rollouts visiting $(s,a)$: $\widehat{Q}(s,a) = \frac{1}{N(s,a)}\sum_{i:(s_t^i,a_t^i)=(s,a)}R_t^i$, where $N(s,a)$ counts the number of rollouts that visit state $s$ and take action $a$. The state value for $s$ is estimated as $\widehat{V}(s) = \frac{1}{N(s)}\sum_{i:s_t^i=s}R_t^i$, where $N(s)=\sum_aN(s,a)$. The per-step advantage is finally derived as the difference between the action value and the state value: $\widehat{A}(s,a) = \widehat{Q}(s,a) - \widehat{V}(s)$.

Fundamentally, the advantage estimation in RTMC is equivalent to that in GiGPO. First,  both methods compute discounted returns for each step and then group steps that share the same state into a group and normalize the discounted returns within each group. While RTMC additionally averages the returns for each $(s,a)$ pair within the state-action group before state-group normalization, this does not introduce any essential difference with GiGPO's step-level advantage estimation. Consequently, the advantages of our method over GiGPO discussed in Appendix \ref{sec:discuss_with_gigpo} apply equally to RTMC.

\section{The Overall Algorithmic Procedure for GRAFT}
\label{sec:overall_alg}

We present the overall algorithmic procedure of our GRAFT in Algorithm \ref{alg:graft}. Our method utilizes graph as the basic representation structure to model state-action transition relationship. Its key characteristic is that it is grounded in reinforcement learning theory rather than heuristics: we perform Bellman iteration on the graph to estimate the state-value of each node, and then derive the step-level advantage from the difference in node values. Furthermore, we extend GAE on the graph to obtain more robust advantage estimation. Additionally, for loss computation, we identify the misalignment between the step-level advantage and the optimization objective in GRPO, thus we further adopt a calibrated optimization objective (Eq.(\ref{eq:grpo_c})) to ensure better alignment.

\begin{algorithm}[ht]
\caption{GRAFT: \textbf{GRA}ph-based \textbf{F}aithful S\textbf{T}ep-level Credit Assignment}
\label{alg:graft}
\begin{algorithmic}[1]
\REQUIRE Policy $\pi_\theta$, environment $\mathcal{E}$, dataset $\mathcal{D}$, group size $N$, discount factor $\gamma$, weighting factor $\lambda$
\FOR{each training iteration}
    \STATE Sample tasks $\{q_k\}$ from $\mathcal{D}$; collect $N$ rollouts per task under $\pi_\theta$ $\rightarrow$ batch $\mathcal{T}$
    \STATE \textbf{// Graph construction}
    \FOR{each task $q_k$}
        \STATE Build trajectory graph $\mathcal{G}_k$ from $\{(s_t, a_t, s_{t+1})\} \subset \mathcal{T}$ via state canonicalization $\phi$
        \STATE Identify sink nodes; assign terminal reward $R_i \in \{0,1\}$ 
    \ENDFOR
    \STATE \textbf{// State-value estimation}
    \FOR{each task $q_k$}
        \STATE Solve the Bellman equation (\ref{eq:bellman_equation}) on $\mathcal{G}_k$ via value iteration $\rightarrow \{V(u)\}$
    \ENDFOR
    \STATE \textbf{// Step-level advantage computation}
    \FOR{each step $(s_t, a_t, s_{t+1}) \in \mathcal{T}$}
        \STATE $A_t \leftarrow \gamma \cdot V(\phi(s_{t+1})) - V(\phi(s_t))$
        \STATE $A_t^{GAE} \leftarrow A_t + \sum_{k=1}^{T-t} (\gamma\lambda)^k \cdot \bar{A}_{t+k}, \quad \bar{A}_{t+k} = \frac{1}{|\mathcal{E}_k(a_t)|} \sum_{(u', a', v') \in \mathcal{E}_k(a_t)} \bigl[\gamma V(\phi(v')) - V(\phi(u'))\bigr]$
        \STATE $\hat{A}_{\text{step}}(s_t,a_t) \leftarrow \text{GroupNorm}(A_t^{GAE};\, \mathcal{G}(\phi(s_t)))$
    \ENDFOR
    \STATE \textbf{// Policy update}
    \STATE Update $\pi_\theta$ via the calibrated optimization objective in Eq.(\ref{eq:grpo_c}) on batch $\mathcal{T}$ using $\{\hat{A}_{i,t}\}$
\ENDFOR
\end{algorithmic}
\end{algorithm}

\section{Proof of Proposition 1}
\label{sec:appendix_proof}

\begin{proof}
We first formalize the distribution under which the advantage estimation error is evaluated. Let $d^\pi$ denote the state-action visitation distribution induced by policy $\pi$ over the time steps used for policy optimization. For a step-level credit signal $X$, its population estimation error is
\begin{equation}
    \mathcal E(X)
=
\mathbb E_{\tau\sim\pi}
\left[
\bigl(
X_t-A^\pi(S_t,A_t)
\bigr)^2
\right]
\end{equation}
where the expectation is taken over both the visitation of $(S_t,A_t)$ and the stochastic continuation of the trajectory after time $t$.

\paragraph{Graph-based TD credit.}
Because the environment is deterministic, each state-action pair $(s,a)$
has a unique successor state
\begin{equation}
    s'=T(s,a)
\end{equation}
The Bellman equation therefore gives
\begin{equation}
    Q^\pi(s,a)
=
r(s,a,s')+\gamma V^\pi(s')
\end{equation}
Using the standard definition of the policy advantage,
\begin{equation}
    A^\pi(s,a)
=
Q^\pi(s,a)-V^\pi(s)
\end{equation}
we obtain
\begin{equation}
A^\pi(s,a)
=
r(s,a,s')
+\gamma V^\pi(s')
-V^\pi(s)
=
X_t^G
\end{equation}
Hence, the population estimation error of our graph-based step-level credit $X_t^G$ follows that
\begin{equation}
\label{eq:graft_advantage_error}
    \mathcal E(X^G)
=
\mathbb E_{\tau\sim\pi}\left[
\bigl(
X_t^G-A^\pi(S_t,A_t)
\bigr)^2
\right]
=
0
\end{equation}

\paragraph{State-grouped step-level credit.}
Recall that the population counterpart of the state-grouped step-level signal in GiGPO \citep{GiGPO} is
\begin{equation}
\label{eq:gigpo_advantage}
    X_t^S
=
G_t-V^\pi(S_t)
\end{equation}
where $ G_t
=
\sum_{k=t}^{T_\tau}
\gamma^{k-t}r_k$ is the sampled return from time $t$.

For a fixed state-action pair $(s,a)$, subtracting the true policy advantage
from Equation~(\ref{eq:gigpo_advantage}) yields
\begin{align}
X_t^S-A^\pi(s,a)
&=
G_t-V^\pi(s)
-
\left[
Q^\pi(s,a)-V^\pi(s)
\right]
\nonumber\\
&=
G_t-Q^\pi(s,a)
\end{align}
By the definition of the action-value function,
\begin{equation}
    Q^\pi(s,a)
=
\mathbb E[G_t\mid S_t=s,A_t=a]
\end{equation}
Therefore,
\begin{align}
\mathbb E\left[
\bigl(
X_t^S-A^\pi(s,a)
\bigr)^2
\mid S_t=s,A_t=a
\right]
&=
\mathbb E\left[
\bigl(
G_t-Q^\pi(s,a)
\bigr)^2
\mid s,a
\right]
\nonumber\\
&=
\operatorname{Var}
\left(
G_t\mid S_t=s,A_t=a
\right)
\end{align}
Taking the expectation over the visitation distribution gives
\begin{equation}
\label{eq:gigpo_advantage_error}
    \mathcal E(X^S)
=
\mathbb E_{(s,a)\sim d^\pi}
\left[
\operatorname{Var}
\left(
G_t\mid S_t=s,A_t=a
\right)
\right]
\end{equation}
In particular, $\mathcal E(X^S)\ge 0$. We next establish when the inequality is strict. Under sparse rewards, suppose that
\begin{equation}
   r_k=0
\quad\text{for }k<T_\tau,
\qquad
R(\tau)=r_{T_\tau} 
\end{equation}
where a failed trajectory receives $R(\tau)=0$ and a successful trajectory
receives $R(\tau)>0$. The return from time $t$ is then
\begin{equation}
    G_t
=
\gamma^{T_\tau-t}R(\tau)
\end{equation}
For every $(s,a)\in\Omega$, both successful and failed continuations occur
with positive probability. Consequently,
\begin{equation}
    \Pr(G_t=0\mid s,a)>0 \qquad \text{and} \qquad \Pr(G_t>0\mid s,a)>0
\end{equation}
Thus, $G_t$ is not constant conditional on $(s,a)$, which implies
\begin{equation}
\label{eq:var_gt}
    \operatorname{Var}(G_t\mid s,a)>0,
\qquad
(s,a)\in\Omega
\end{equation}
If $\Omega$ has positive visitation probability $d^\pi(\Omega)>0$, then Equations~(\ref{eq:gigpo_advantage_error}) and~(\ref{eq:var_gt}) imply $\mathcal E(X^S)>0$.

\paragraph{Trajectory-level credit.}
Let
\begin{equation}
   b_E
=
\mathbb E_{\tau \sim \pi}[R(\tau)\mid S_0=s_0] 
\end{equation}
denotes the population trajectory-level baseline for the fixed task. The
population counterpart of the trajectory-level credit credit in GRPO is
\begin{equation}
    X_t^E
=
R(\tau)-b_E
\end{equation}
Although the same trajectory-level signal is assigned to every step in the
trajectory, the desired target at step $t$ is the state-dependent advantage
$A^\pi(S_t,A_t)$. For a fixed state-action pair $(s,a)$, define
\begin{equation}
    \mu_E(s,a)
=
\mathbb E[X_t^E\mid S_t=s,A_t=a]
\end{equation}
The estimation error can be decomposed as
\begin{equation}
X_t^E-A^\pi(s,a)
=
\left[
X_t^E-\mu_E(s,a)
\right]
+
\left[
\mu_E(s,a)-A^\pi(s,a)
\right]
\end{equation}
Squaring both sides and taking the conditional expectation gives
\begin{equation}
\label{eq:grpo_advantage_error}
\mathbb E\left[
\bigl(
X_t^E-A^\pi(s,a)
\bigr)^2
\mid S_t=s,A_t=a
\right]=
\operatorname{Var}(X_t^E\mid s,a)
+
\left[
\mu_E(s,a)-A^\pi(s,a)
\right]^2
\end{equation}
The above derivation is because the cross term has conditional expectation zero.

Since $b_E$ is a constant for the fixed task, thus
\begin{equation}
\label{eq:var_episode}
   \operatorname{Var}(X_t^E\mid s,a)
=
\operatorname{Var}(R(\tau)\mid s,a) 
\end{equation}
Substituting Equation~(\ref{eq:var_episode}) into Equation~(\ref{eq:grpo_advantage_error}) yields
\begin{equation}
\mathbb E\left[
\bigl(
X_t^E-A^\pi(s,a)
\bigr)^2
\mid s,a
\right]
=
\operatorname{Var}(R(\tau)\mid s,a)
+
\left[
\mu_E(s,a)-A^\pi(s,a)
\right]^2
\end{equation}
Taking the expectation over $d^\pi$ gives
\begin{equation}
\mathcal E(X^E)
=
\mathbb E_{(s,a)\sim d^\pi}
\left[
\operatorname{Var}(R(\tau)\mid s,a)
\right]
+
\mathbb E_{(s,a)\sim d^\pi}
\left[
\left(
\mu_E(s,a)-A^\pi(s,a)
\right)^2
\right]
\end{equation}
Both terms on the right-hand side are non-negative, and hence $\mathcal E(X^E)\ge 0$.

For every $(s,a)\in\Omega$, both successful and failed trajectories occur
with positive probability. Therefore,
\begin{equation}
    \Pr(R(\tau)=0\mid s,a)>0 \qquad \text{and} \qquad \Pr(R(\tau)>0\mid s,a)>0
\end{equation}
It follows that
\begin{equation}
    \operatorname{Var}(R(\tau)\mid s,a)>0,
\qquad
(s,a)\in\Omega
\end{equation}
If $d^\pi(\Omega)>0$, Equation~(\ref{eq:grpo_advantage_error}) therefore implies $\mathcal E(X^E)>0$.

\paragraph{Comparison.}
Equation~(\ref{eq:graft_advantage_error}) establishes that $\mathcal E(X^G)=0$, Equations~(\ref{eq:gigpo_advantage_error}) and~(\ref{eq:grpo_advantage_error}) establish that $\mathcal E(X^S)\ge0,
\mathcal E(X^E)\ge0 \nonumber$. Therefore,
\begin{equation}
\label{eq:lower_eq}
    \mathcal E(X^G)
\le
\min\left\{
\mathcal E(X^S),
\mathcal E(X^E)
\right\}
\end{equation}

Furthermore, if $\Omega$ has positive visitation probability under $\pi$,
Equations~(\ref{eq:var_gt}) and~(\ref{eq:var_episode}) give 
\begin{equation}
\label{eq:lower}
    \mathcal E(X^S)>0, \qquad \mathcal E(X^E)>0
\end{equation}
Combining Equations~(\ref{eq:lower_eq}) and~(\ref{eq:lower}), we obtain the strict inequality
\begin{equation}
    \mathcal E(X^G)
<
\min\left\{
\mathcal E(X^S),
\mathcal E(X^E)
\right\}
\end{equation}
This completes the proof.
\end{proof}

\paragraph{Remark.}
Proposition~\ref{prop:advantage_error} compares the population credit
signals underlying the three methods. In practice, GRAFT replaces
$V^\pi$ with the graph-based estimator $\widehat V$. For
\[
\widehat X_t^G
=
r_t+\gamma\widehat V(S_{t+1})-\widehat V(S_t)
\]
its deviation from the population graph-based TD credit satisfies
\[
\widehat X_t^G-X_t^G
=
\gamma\left[
\widehat V(S_{t+1})-V^\pi(S_{t+1})
\right]
-
\left[
\widehat V(S_t)-V^\pi(S_t)
\right]
\]
Thus, the empirical graph-based step-level advantage estimator approaches the zero-error population target as the graph-based value estimates become more accurate.

\begin{proposition}[Lower Finite-Sample Advantage Estimation Error]
\label{prop:advantage_error2}
Consider $N\ge2$ independent length-$T$ rollouts from the same initial
state under a fixed policy $\pi$ in a deterministic environment with sparse terminal-state rewards $R\in[0,1]$, and $\gamma\in(0,1)$. Let
\[
X_t^G=\gamma \widehat{V}(S_{t+1})-\widehat{V}(S_t),\quad
X_t^S=G_t-\overline G(S_t),\quad
X_t^E=R-\overline R
\]
be the unnormalized graph-based, state-grouped, and trajectory-level credits of GRAFT, GiGPO, and GRPO, respectively. $\widehat{V}$ is the empirical Bellman solution with $\widehat{V}(S_T)=R$, $G_t=\gamma^{T-t}R$,
$\overline G(u)$ averages over all visits to $u$ across trajectories
and time steps, and $\overline R$ is the whole-group mean.
Let $e_t=\widehat{V}(S_t)-V^\pi(S_t)$ denotes the state-value estimation error. Define the mean-squared advantage estimation error as
$\mathcal E(X)=\mathbb E_{\tau \sim \pi}[(X_t-A^\pi(S_t,A_t))^2]$,
where the expectation is taken over the sampled rollout group and an action selected uniformly from that group. If
\begin{equation}
\label{eq:finite_sample_accuracy}
\begin{aligned}
\|\gamma e_{t+1}-e_t\|_2
+\max\!\left\{
\|\overline G(S_t)-V^\pi(S_t)\|_2,
\sqrt{\frac{\operatorname{Var}(R)}{N}}
\right\}\le
\sqrt{\mathbb E[\operatorname{Var}(G_t\mid S_t,A_t)]}
\end{aligned}
\end{equation}
then
\begin{equation}
\label{eq:finite_sample_comparison}
\mathcal E(X^G)\le
\min\{\mathcal E(X^S),\mathcal E(X^E)\}
\end{equation}
Both comparisons are strict if Eq.~\ref{eq:finite_sample_accuracy}
is strict.
\end{proposition}

\begin{proof}
We evaluate all three credit signals on the same rollout group.
Independently of the group, select $i$ uniformly from
$\{1,\ldots,N\}$ and $t$ uniformly from $\{0,\ldots,T-1\}$,
and write
$(S_t,A_t,S_{t+1},R)=(S_{i,t},A_{i,t},S_{i,t+1},R_i)$.
For any random variable $Z$, let $\|Z\|_2=(\mathbb E[Z^2])^{1/2}$.

Following the terminal-value convention, $V^\pi(S_{i,T})=R_i$ and
\begin{equation}
\label{eq:revised_return_target}
G_{i,t}=\gamma^{T-t}R_i,\qquad
Q^\pi(S_t,A_t)=\mathbb E[G_t\mid S_t,A_t],\qquad
A^\pi=Q^\pi-V^\pi
\end{equation}
 Define the state visit group
\begin{equation}
\label{eq:revised_cross_time_group}
\mathcal I(u)=\{(j,k):1\le j\le N,\ 0\le k<T,\ S_{j,k}=u\}
\end{equation}
The empirical baselines are
\begin{equation}
\label{eq:revised_empirical_baselines}
\overline G(u)=\frac1{|\mathcal I(u)|}
\sum_{(j,k)\in\mathcal I(u)}G_{j,k},\qquad
\overline R=\frac1N\sum_{j=1}^N R_j
\end{equation}
All occurrences are included, including the selected occurrence and
repeated visits within one trajectory. The evaluated group is always
nonempty. For brevity in the proof, set
\begin{equation}
\label{eq:revised_noise_baseline_errors}
\begin{aligned}
\sigma_G&=\sqrt{\mathbb E[\operatorname{Var}(G_t\mid S_t,A_t)]},\\
\sigma_R&=\sqrt{\mathbb E[\operatorname{Var}(R\mid S_t,A_t)]},\\
\eta_S&=\|\overline G(S_t)-V^\pi(S_t)\|_2,\qquad
\eta_E=\sqrt{\operatorname{Var}(R)/N}
\end{aligned}
\end{equation}

\paragraph{Graph-based TD credit.}
Because the environment is deterministic, the true Bellman equation
at an executed transition gives
\begin{equation}
A^\pi(S_t,A_t)=\gamma V^\pi(S_{t+1})-V^\pi(S_t)
\end{equation}
Consequently,
\begin{equation}
\begin{aligned}
X_t^G-A^\pi(S_t,A_t)
&=\gamma\bigl[\widehat{V}(S_{t+1})-V^\pi(S_{t+1})\bigr]-\bigl[\widehat{V}(S_t)-V^\pi(S_t)\bigr]\\
&=\gamma e_{t+1}-e_t
\end{aligned}
\end{equation}
Thus the finite-sample graph-credit error is exactly
\begin{equation}
\label{eq:revised_graph_error}
\sqrt{\mathcal E(X^G)}=\|\gamma e_{t+1}-e_t\|_2
\end{equation}

\paragraph{State-grouped step-level credit.}
Subtracting the true advantage yields
\begin{equation}
\begin{aligned}
X_t^S-A^\pi(S_t,A_t)
={}&\bigl[G_t-Q^\pi(S_t,A_t)\bigr]
-\bigl[\overline G(S_t)-V^\pi(S_t)\bigr]
\end{aligned}
\end{equation}
By the definition of $Q^\pi$,
\begin{equation}
\|G_t-Q^\pi(S_t,A_t)\|_2^2
=\mathbb E[\operatorname{Var}(G_t\mid S_t,A_t)]
=\sigma_G^2
\end{equation}
Applying the reverse triangle inequality gives
\begin{equation}
\label{eq:revised_step_error_bound}
\sqrt{\mathcal E(X^S)}\ge(\sigma_G-\eta_S)_+,
\qquad (x)_+=\max\{x,0\}
\end{equation}
The argument permits arbitrary dependence between the sampled return
and its empirical state-grouped baseline. In particular, it applies
to cross-time groups and repeated visits without treating them as
independent samples.

\paragraph{Trajectory-level credit.}
Let $b_E=\mathbb E[R]$ and define the population-centered signal
$Y_t^E=R-b_E$. Its conditional bias--variance decomposition is
\begin{equation}
\begin{aligned}
\mathbb E[(Y_t^E-A^\pi(S_t,A_t))^2]
={}&\mathbb E[\operatorname{Var}(R\mid S_t,A_t)]\\
&+\mathbb E\!\left[
\bigl(\mathbb E[R\mid S_t,A_t]-b_E-A^\pi(S_t,A_t)\bigr)^2
\right]
\end{aligned}
\end{equation}
Hence $\|Y_t^E-A^\pi(S_t,A_t)\|_2\ge\sigma_R$.
Since the true state includes $t$ and $G_t=\gamma^{T-t}R$,
\begin{equation}
\sigma_G^2
=\mathbb E[\gamma^{2(T-t)}\operatorname{Var}(R\mid S_t,A_t)]
\le\sigma_R^2
\end{equation}
The empirical baseline includes the selected trajectory, so
\begin{equation}
X_t^E=R_i-\overline R
=\left(1-\frac1N\right)R_i-\frac1N\sum_{j\ne i}R_j
\end{equation}
The $N$ complete rollouts are independent and identically distributed.
Therefore,
\begin{equation}
\|X_t^E-Y_t^E\|_2^2
=\mathbb E[(\overline R-b_E)^2]
=\frac{\operatorname{Var}(R)}N
=\eta_E^2
\end{equation}
A further application of the reverse triangle inequality gives
\begin{equation}
\label{eq:revised_episode_error_bound}
\sqrt{\mathcal E(X^E)}
\ge(\sigma_R-\eta_E)_+
\ge(\sigma_G-\eta_E)_+
\end{equation}
No independence between $R_i$ and $\overline R$ is assumed.

\paragraph{Comparison.}
Set $\eta=\max\{\eta_S,\eta_E\}$.
Equations~\ref{eq:revised_step_error_bound} and
\ref{eq:revised_episode_error_bound} imply
\begin{equation}
\min\!\left\{\sqrt{\mathcal E(X^S)},\sqrt{\mathcal E(X^E)}\right\}
\ge(\sigma_G-\eta)_+
\end{equation}
Under condition~\ref{eq:finite_sample_accuracy},
\begin{equation}
\sqrt{\mathcal E(X^G)}
=\|\gamma e_{t+1}-e_t\|_2
\le\sigma_G-\eta
\end{equation}
The right-hand side is nonnegative. Squaring proves
\begin{equation}
\mathcal E(X^G)\le
\min\{\mathcal E(X^S),\mathcal E(X^E)\}
\end{equation}
If condition~\ref{eq:finite_sample_accuracy} is strict, the preceding
comparison of square roots is strict and $\sigma_G-\eta>0$, yielding
strict inequalities against both baselines. This completes the proof.
\end{proof}

\paragraph{Remark.}
Condition~\ref{eq:finite_sample_accuracy} relates the admissible graph-credit error to the accuracy of the empirical
baselines. For a fixed level of conditional return variability, smaller baseline errors increase the margin
$\sigma_G-\max\{\eta_S,\eta_E\}$, allowing a larger graph-credit estimation error while preserving the comparison. With correct state aggregation and adequate coverage of
relevant states, additional independent rollouts can improve both the empirical baselines and the transition statistics
used for graph-value estimation. These improvements act in complementary directions: more accurate baselines enlarge
the admissible error margin, while more accurate graph values can reduce the graph-credit error. Thus, the condition describes a regime in which reliable cross-trajectory statistics support more accurate credit assignment, without requiring exact value estimates.

\section{Additional Experimental Analysis}
\label{sec:additional_experiments}

\subsection{Hyperparameter Ablation Studies}

Tab.\ref{tab:hyperparameter_results} investigates the sensitivity of our method to the discount factor $\gamma$ and the weighting factor $\lambda$. For $\gamma$, we observe that larger values yield better performance. This is expected in long-horizon agentic tasks, where a higher $\gamma$ facilitates the backward propagation of sparse terminal rewards to earlier steps; conversely, lower values may cause initial steps to receive insufficient reward signals for effective learning. Regarding $\lambda$, a larger value indicates a longer backward-looking horizon during Graph GAE estimation, incorporating more k-hop step estimators to refine the single-step estimator. 
Since the optimal lookback depth varies with task complexity, $\lambda$ is typically tuned per benchmark. Overall, the results demonstrate robust performance across a reasonable range of hyperparameters.

\begin{table*}[ht]
\centering
\caption{Hyperparameter ablation studies on the discount factor $\gamma$ and the weighing factor $\lambda$ in Graph GAE (Eq.(\ref{eq:graph_gae})).}
\label{tab:hyperparameter_results}
\resizebox{0.6\linewidth}{!}{
\begin{tabular}{c | cc | c | cc}
\toprule
$\gamma$ & ALFWorld & WebShop & $\lambda$ & ALFWorld & WebShop\\
\midrule
0.8
  & 95.6
  & 76.3
  & 0.5
  & 95.6
  & 79.7
   \\
0.9
  & 96.1
  & 79.2
  & 0.6
  & 97.2
  & 81.2
  \\
0.95
  & 96.4
  & 78.1
  & 0.8
  & 96.4
  & 82.8
  \\
0.97
  & 96.9
  & 80.5
  & 0.9
  & 96.4
  & 77.3
  \\
 0.99
  & 96.9
  & 81.0
  & 0.95
  & 97.7
  & 81.2
   \\
\bottomrule
\end{tabular}}
\end{table*}

\subsection{Node Group Size}
\label{sec:node_group_size}

\begin{figure}[htbp]
\centering
\begin{minipage}[b]{0.48\textwidth}
\centering
\includegraphics[width=\textwidth]{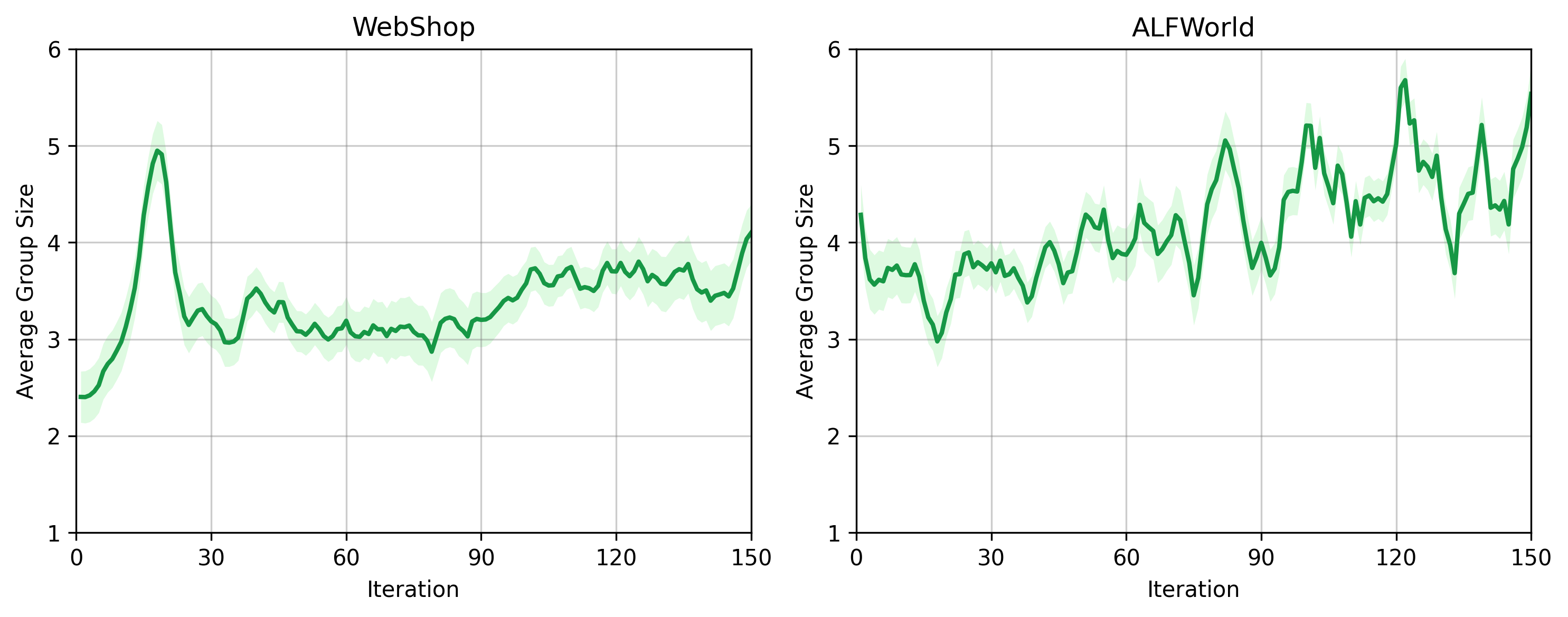} 
\caption{Dynamics of average node group size during the training process.}
\label{fig:group_size_dynamics}
\end{minipage}
\hfill 
\begin{minipage}[b]{0.48\textwidth}
\centering
\includegraphics[width=\textwidth]{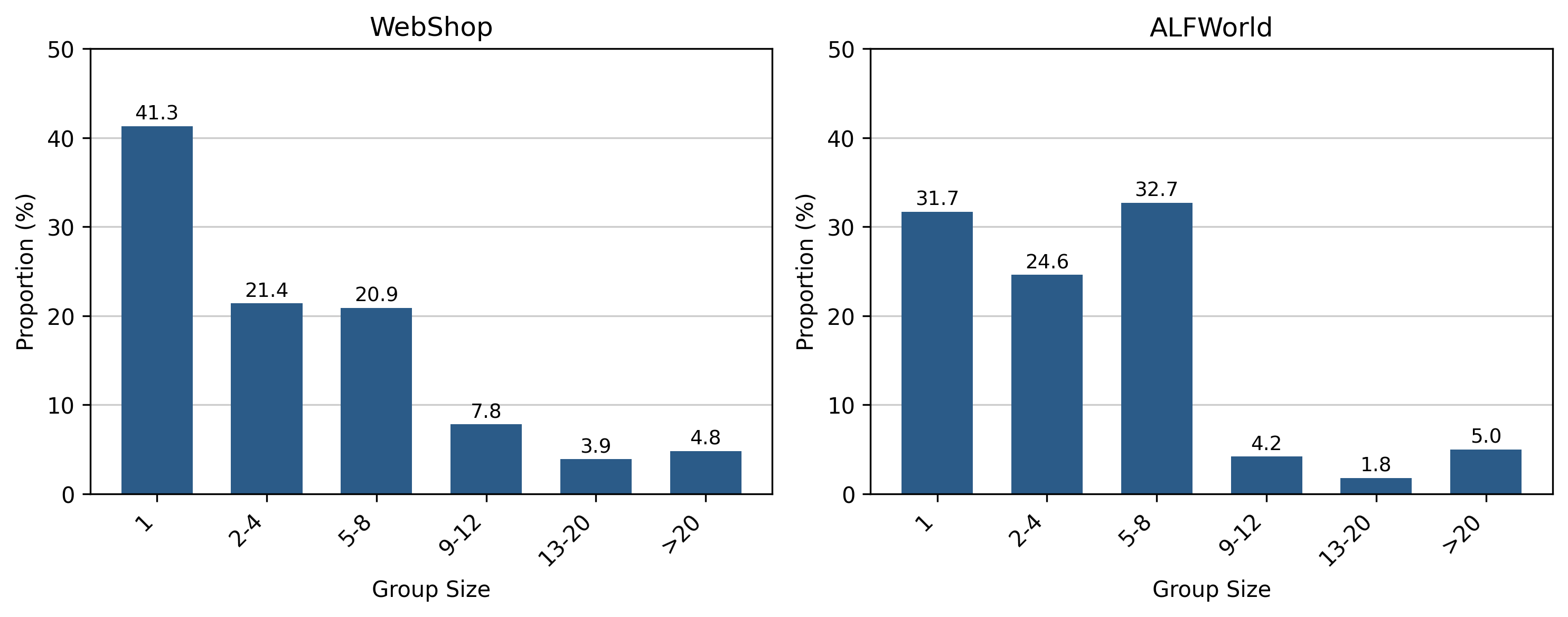} 
\caption{The distributions of node group size during training.}
\label{fig:group_size_distributions}
\end{minipage}
\end{figure}

Here, we present the dynamics of the average node group size (\emph{i.e.}, the number of outgoing edges per node) during training and the distributions of group size. As shown in Fig.\ref{fig:group_size_dynamics}, the average group size for WebShop and ALFWorld is around 4 and 5. This indicates that many environmental observations recur in distinct trajectories yet lead to divergent outcomes, underscoring the necessity of aggregating trajectories into a unified trajectory graph. In WebShop, the average group size initially increases, suggesting that the agent’s decisions converge
during early training, leading to more frequent state revisits. Subsequently, the group size decreases
because fine-grained credit assignment helps the agent to eliminate redundant steps, leading to shorter trajectories and fewer total observations. In ALFWorld, the average group size increases, this is because that ALFWorld tasks inherently require repetitive actions, forcing the agent to revisit identical states. This inherent repetition amplifies the upward trend in group size, outweighing the downward trend caused by redundant step removal. Furthermore, as shown in Fig.\ref{fig:group_size_distributions}, on both WebShop and ALFWorld benchmarks, more than 60\% of steps fall into groups of 2 or larger, meaning the more faithful step-level advantage estimation benefit from the shared state group-aggregation mechanism.

\subsection{Information Propagation Degree}
\label{sec:information_degree}

\begin{figure}[htbp]
\centering
\begin{minipage}[b]{0.48\textwidth}
\centering
\includegraphics[width=\textwidth]{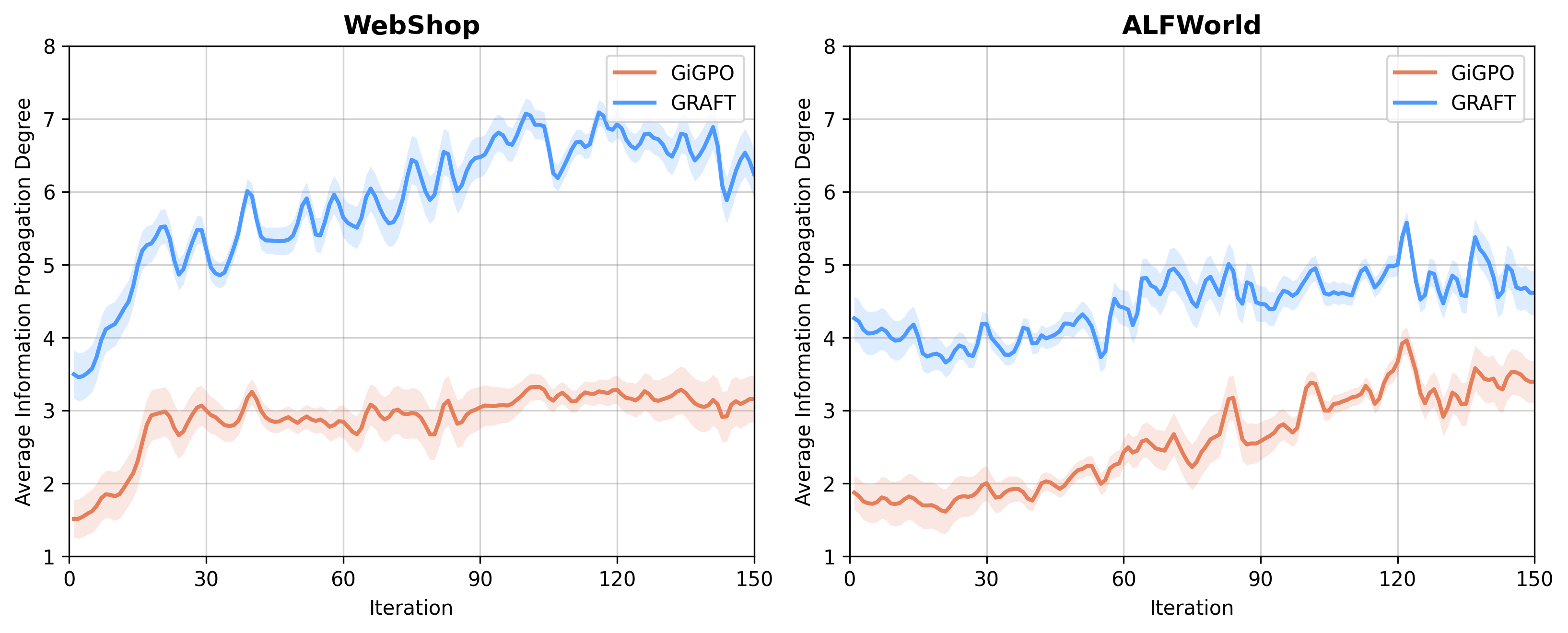} 
\caption{Dynamics of average information propagation degree during training.}
\label{fig:degree_dynamics}
\end{minipage}
\hfill 
\begin{minipage}[b]{0.48\textwidth}
\centering
\includegraphics[width=\textwidth]{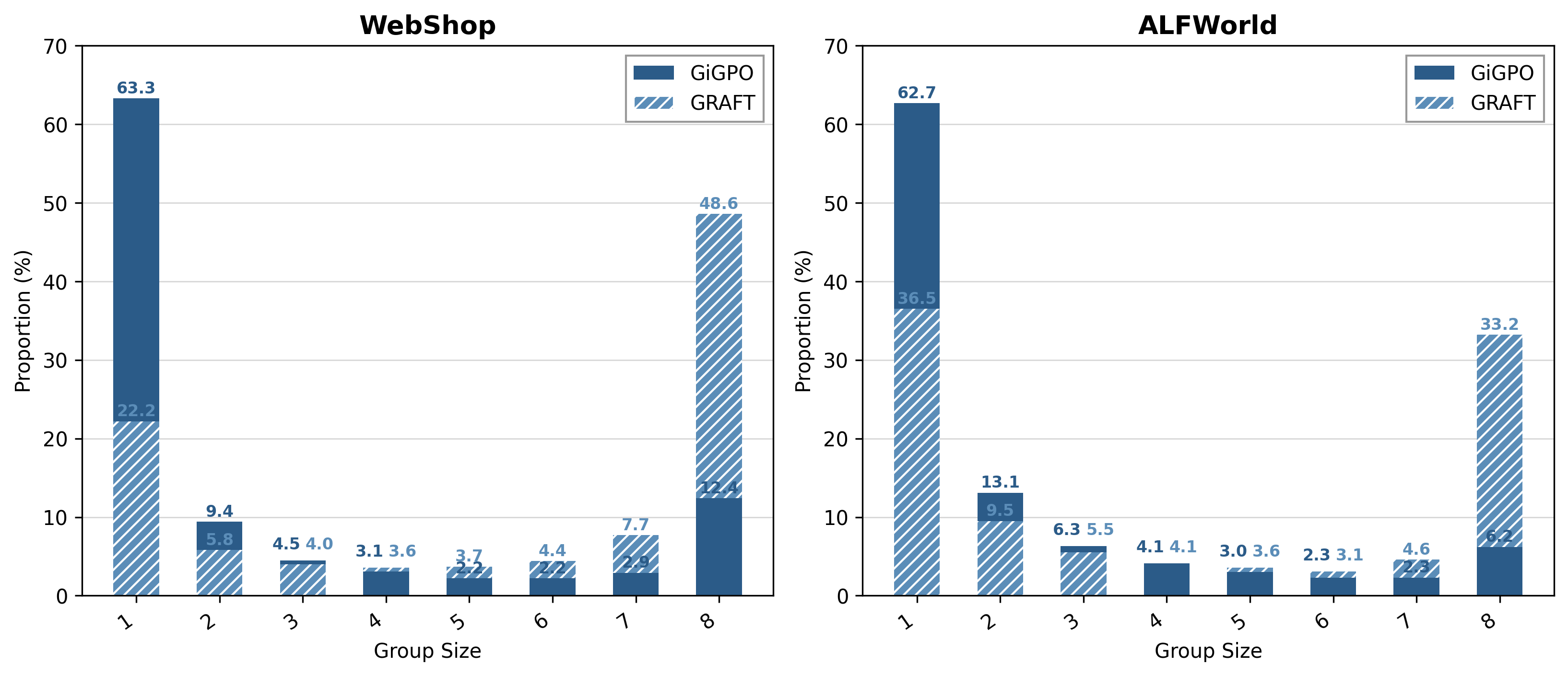} 
\caption{The distributions of information propagation degree.}
\label{fig:degree_distributions}
\end{minipage}
\end{figure}

We further analyze the dynamics and distributions of the information propagation degree during training. The information propagation degree is defined as the number of distinct trajectories contributing reward information to the value estimation of a node (or state), as starting from one node in the graph can lead to multiple different trajectories. As shown in Fig.\ref{fig:degree_dynamics}, our method consistently achieves a higher information propagation degree than GiGPO across training. This indicates that graph-based state-value estimation will aggregate reward signals from a broader set of trajectories for each node. In contrast, GiGPO's simple state grouping restricts reward propagation to the existing trajectory, preventing cross-trajectory credit flow. The results further support the analysis in Appendix \ref{sec:discuss_with_gigpo} that: graph-based state-value estimation will be more reliable as it can leverage more reward information. Furthermore, in Fig.\ref{fig:degree_distributions}, up to 60\% of the steps in GiGPO receive reward information from only a single trajectory, whereas in our GRAFT, this proportion significantly drops to 20\%.

\subsection{Additional Results}
\label{sec:additional_results}

\textbf{Performance based on Qwen3 model series.} In Tab.\ref{tab:main_results}, for the convenience of direct and fair comparison with prior methods, we conducted experiments based on the Qwen2.5-Instruct model series. Here, to further validate generalizability, we further extend our evaluation to the Qwen3 family (\emph{e.g.}, Qwen3-4B and Qwen3-8B, we reproduce GiGPO and GraphGPO for comparison). The results in Tab.\ref{tab:qwen3_results} show that our method significantly outperforms GiGPO and GraphGPO on both Qwen3-4B and Qwen3-8B models. These evaluations confirm that our method is not tied to a specific base model, but can be effectively transferred to other base models.

\begin{table*}[ht]
\centering
\caption{Performance on ALFWorld and WebShop. These results are based on the Qwen3 model series, \emph{e.g.}, Qwen3-4B and Qwen3-8B.}
\label{tab:qwen3_results}
\resizebox{\linewidth}{!}{
\begin{tabular}{ll ccccccc cc}
\toprule
\multirow{2}{*}{\textbf{Type}} & \multirow{2}{*}{\textbf{Method}}
  & \multicolumn{7}{c}{\textbf{ALFWorld}}
  & \multicolumn{2}{c}{\textbf{WebShop}} \\
\cmidrule(lr){3-9} \cmidrule(lr){10-11}
 & & \textbf{Pick} & \textbf{Look} & \textbf{Clean} & \textbf{Heat} & \textbf{Cool} & \textbf{Pick2} & \textbf{All}
   & \textbf{Score} & \textbf{Succ.} \\
\midrule
\multicolumn{11}{c}{\textit{Qwen3-4B}} \\
\midrule

RL Training & GiGPO              & 100.0 & 81.8 & 77.3 & 54.5 & 72.0 & 84.0 & 82.0 & 84.1 & 70.6 \\

RL Training & GraphGPO              & 100.0 & 70.0 & 83.3 & 78.6 & 92.9 & 69.6 & 85.9 & 85.8 & 78.9 \\
\cmidrule(lr){1-11}

\rowcolor{paleblue} RL Training & \textbf{GRAFT (Ours)} & 100.0 &  90.0 & 91.7 & 92.9 & 100.0 & 91.3 & 95.3 & 89.6 & 82.8 \\
\midrule
\multicolumn{11}{c}{\textit{Qwen3-8B}} \\
\midrule
RL Training & GiGPO              & 91.7 & 73.3 & 82.8 & 85.7 & 89.2 & 81.2 & 86.7 & 82.7 & 71.1 \\

RL Training & GraphGPO              & 95.3 & 90.0 & 83.3 & 78.6 & 100.0 & 95.7 & 91.4 & 88.0 & 78.9 \\
\cmidrule(lr){1-11}

\rowcolor{paleblue} RL Training & \textbf{GRAFT (Ours)} & 100.0 & 100.0 & 100.0 & 87.5 & 90.5 & 94.7 & 96.1 & 90.9 & 80.5 \\
\bottomrule
\end{tabular}}
\end{table*}

\subsection{Case Study}
\label{sec:case_study}
\begin{figure*}[htbp]
\centering
\includegraphics[width=\textwidth]{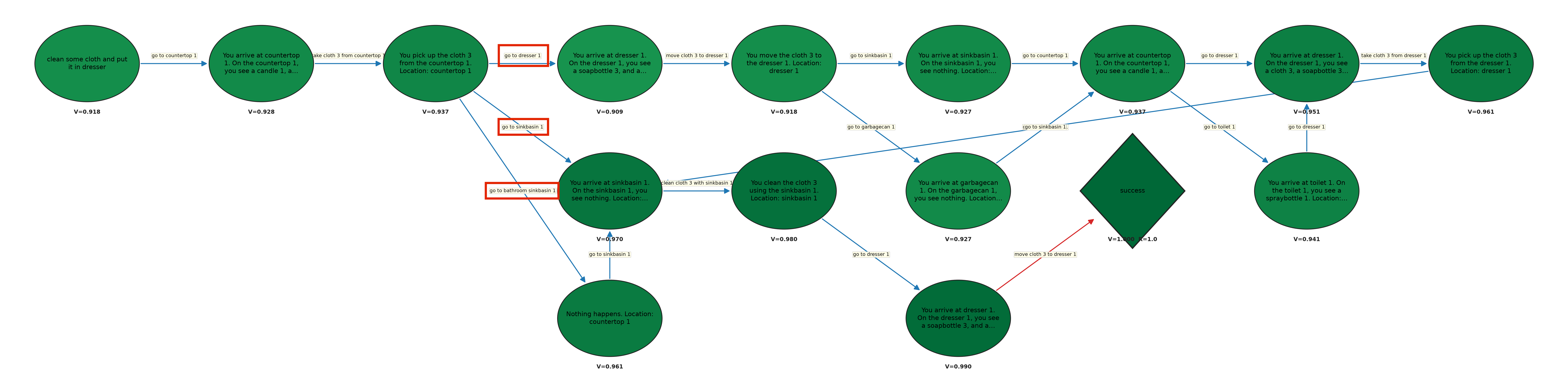} 
\caption{A case study illustrating the constructed graph during training on ALFWorld. In this task, all trajectories finally reach success.}
\label{fig:alfworld_graph1}
\end{figure*}

 To provide a more intuitive understanding of why our method can yield more faithful step-level advantage estimates, we visualize some case graphs constructed during training on ALFWorld. The graphs are shown in Fig.\ref{fig:alfworld_graph1} and \ref{fig:alfworld_graph2}. In Fig.\ref{fig:alfworld_graph1}, all trajectories ultimately lead to success, whereas Fig.\ref{fig:alfworld_graph2} contains a mix of successful and failing trajectories. Node color reflects state-value: red indicates lower values, while green indicates higher values. The visualization reveals that state-values increase with proximity to successful terminal states and remain low for nodes from which success is unreachable. Furthermore, nodes with mixed outcomes—where some trajectories succeed and others fail—exhibit intermediate state-values, reflecting the expected value under the current policy.

In Fig.\ref{fig:alfworld_graph1}, the task is to clean some cloth and put it in dresser. Although all trajectories ultimately succeed, our method still effectively identifies specific erroneous steps. Consider the actions highlighted by the red boxes, the preceding state is ``you pick up the cloth 3 from the countertop 1''. According to the task goal, after obtaining the cloth, it needs to be cleaned first and then placed in the dresser. However, the action ``go to dresser 1'' is to directly put the cloth in the dresser, which results in deviating from the goal. Subsequently, with making multiple corrective actions, the trajectory ultimately succeeds. In the graph, this action leads to a lower-value state, resulting in a negative advantage to successfully capture this erroneous action. The action ``go to sinkbasin 1'' instead results in successful cleaning before placing in the dresser, thereby transitioning the agent into a higher-value state. Meanwhile, ``go to bathroom sinkhole 1'' is an invalid action, and it takes an additional step to reach the ``sinkhole 1'' state. It can be found that the computed advantage of ``go to bathroom sinkhole 1'' is also slightly lower than that of ``go to sinkhole 1''. Overall, this qualitative analysis demonstrates that our method can discriminate among individual actions based on their true contributions, whereas trajectory-level advantage estimation assigns uniform credit to all steps in a successful trajectory, failing to capture fine-grained action quality.

\begin{figure*}[htbp]
\centering
\includegraphics[width=\textwidth]{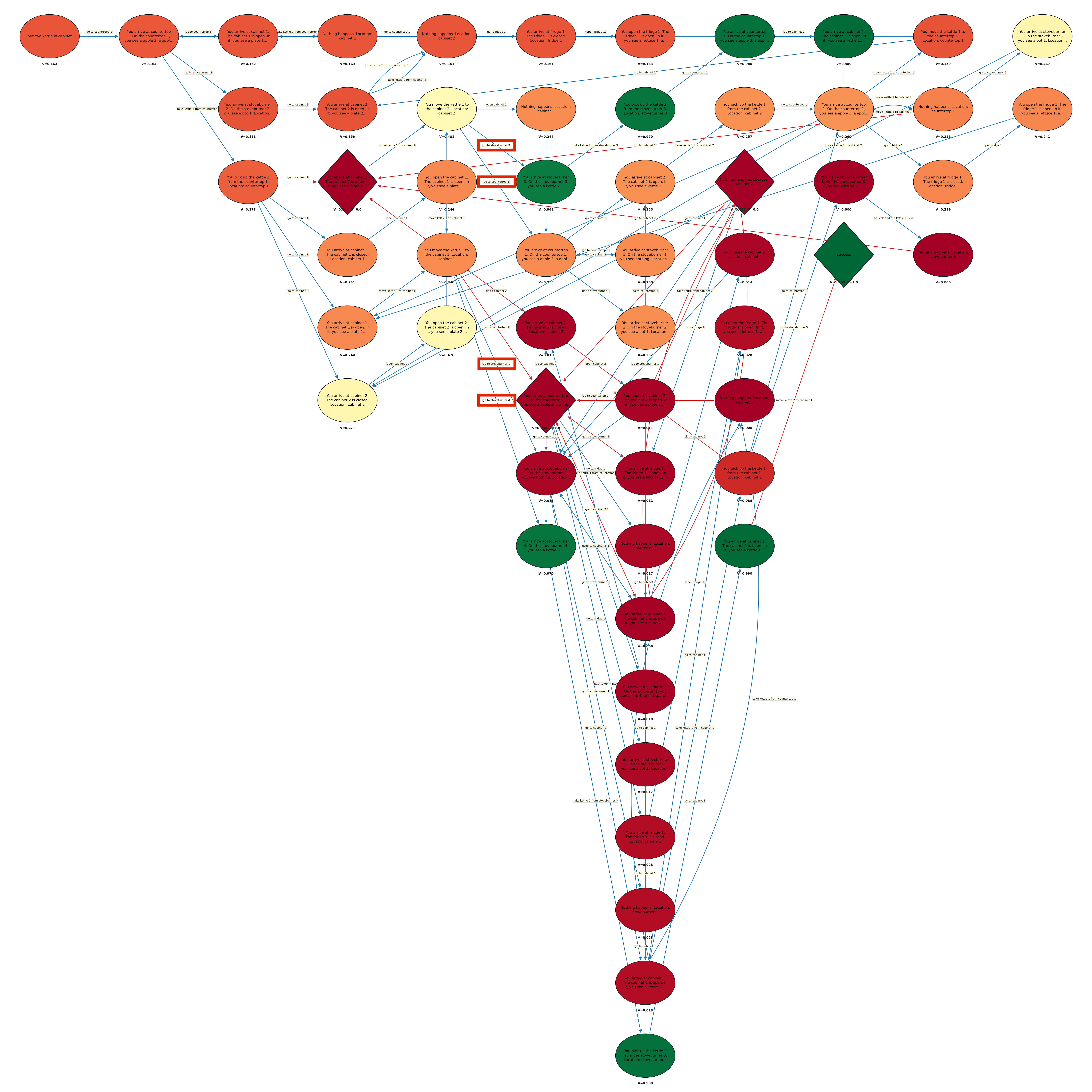} 
\caption{A case study illustrating the constructed graph during training on ALFWorld. This graph contains a mix of successful and failing trajectories.}
\label{fig:alfworld_graph2}
\end{figure*}

In Fig.\ref{fig:alfworld_graph2}, the task is to put two kettle in cabinet. Let's first focus on the actions marked in the upper red boxes. The preceding state is ``you move the kettle 1 to the cabinet 2''. Now that a kettle has been placed in the cabinet, when action ``go to stoveburner 3'' is taken, it enables the agent to locate the second kettle. Consequently, this action leads to a significant increase in the subsequent state values. In contrast, ``go to countertop 1'' returns the agent to a location where kettle 1 was originally found; since the kettle has already been moved, no new kettle can be obtained there. This results in a marked decrease in downstream state values. A similar situation appears in the lower red boxes: ``go to stoveburner 1'' enters a state without a kettle, resulting in a low value, whereas ``go to stoveburner 4'' allows the agent to find kettle 3, again producing a substantial increase in value. This example demonstrates that, in tasks containing both successful and failed trajectories, our method effectively identifies critical decision points, assigning higher advantages to steps pivotal to success and lower advantages to those leading toward failure.

\section{Experiment Details}
\label{sec:experiment_details}

\subsection{Descriptions of Benchmarks}

ALFWorld is a text-based environment aligned with the ALFRED embodied AI benchmark, which is designed to assess an agent's ability to perform long-horizon, multi-step decision-making tasks. WebShop is a complex web-based interactive environment that tests LLM agents in realistic online shopping scenarios, requiring them to navigate a realistic web interface to purchase items matching user specifications. The SearchQA benchmark comprises several widely-used search-augmented QA datasets, including single-hop QA (NQ \citep{NQ}, TriviaQA \citep{triviaQA}, PopQA \citep{popQA}) and multi-hop QA (HotpotQA \citep{hotpotQA}, 2Wiki \citep{2wiki}, MuSiQue \citep{musique}, Bamboogle \citep{bamboogle}).

\subsection{Training Details}
\label{sec:training_details}

To ensure a direct and fair comparison with previous methods, we follow the training and evaluation configurations from GiGPO \citep{GiGPO} across all benchmarks. The detailed settings for each benchmark are described below:

\textbf{Hyperparameters for ALFWorld}. In ALFWorld, since states are predefined strings provided by the environment, state aggregation for building graph node is based on exact matching. We set the maximum prompt and response lengths to 2048 and 512 tokens, respectively. Each trajectory allows up to 50 environment steps. Training runs for 150 steps with a policy learning rate of $1e^{-6}$. We employ a rule-based reward scheme: $+10$ for success, 0 for failure. To handle invalid actions
generated by the agent, we apply a reward penalty of $-0.1$. For all group-based RL methods (GRPO, GiGPO, HCAPO, GraphGPO), we use a group size of $G = 8$ and sample 16 different tasks per rollout step, yielding $16 \times 8 = 128$ parallel environments. In contrast, PPO uses 128 independent environments for fair comparison. The rollout and validation temperatures are set to 1.0 and 0.4, respectively. The mini-batch size is 256, and the KL-divergence loss coefficient $\beta_{KL}$ is 0.01.

\textbf{Hyperparameters for WebShop}. In WebShop, state aggregation is also based on exact matching. We set the maximum prompt and response lengths to 5120 and 512 tokens, respectively, with each episode capped at 15 environment steps. Training runs for 150 steps with a policy learning rate of $1e^{-6}$. We adopt a rule-based reward scheme: assigning $+10$ for success and 0 for failure. Invalid actions are penalized with a reward of $-0.1$. Consistent with ALFWorld, all group-based RL methods use a group size of $G = 8$ and sample 16 tasks per rollout step, yielding 128 parallel environments. PPO uses 128 distinct environments for rollouts. The rollout and validation temperatures are set to 1.0 and 0.4, respectively. The mini-batch size is 64, and $\beta_{kL}$ is 0.01.

\textbf{Hyperparameters for SearchQA}. In SearchQA, as there are no predefined state strings provided by the environment, the state aggregation is based on similarity-based matching, with the similarity threshold $\tau$ set to 0.9. state aggregation for building graph node is based on exact matching. We set the maximum prompt and response lengths to 4096 and 512 tokens, respectively. The maximum number of turns is set to 4. The learning rate is $1e^{-6}$ for the policy model. We employ a
rule-based reward scheme: assigning $+1$ for success and 0 for failure. Invalid actions are penalized with a reward of
$-0.01$. We set the training data size to 256 and use a group size of $G = 5$. The rollout and validation temperatures are set to 1.0 and 0.0, respectively. The mini-batch size is 512, and $\beta_{KL}$ is 0.001.

\subsection{Training Prompts}
\label{sec:training_prompts}

The prompts we use for LLM agents are constructed using Python-style string formatting, where placeholders enclosed in
curly braces (\emph{e.g.}, {task description}, {step count}, and {current observation}) serve as semantic slots dynamically populated at runtime via Python’s .format() function. To enrich the agent’s context, we incorporate historical information and set the history length to 2 for ALFWorld and WebShop, while retaining the full history for search-augmented QA experiments.

The $<$think$>$ $<$/think$>$ block instructs the agent to perform explicit step-by-step reasoning, thereby facilitating
Chain-of-Thought (CoT) deliberation. The $<$action$>$ $<$/action$>$ block is used to clearly indicate the final
action decision. For the search agent, reasoning traces are enclosed in $<$think$>$ $<$/think$>$, search queries in
$<$search$>$ $<$/search$>$, and final answers in $<$answer$>$ $<$/answer$>$ tags. Retrieved evidence from
the search engine is presented within $<$information$>$ $<$/information$>$ tags. The detailed prompt templates for each benchmark are provided in Fig.\ref{fig:prompt-alfworld} (ALFWorld), Fig.\ref{fig:prompt-webshop} (Webshop), Fig.\ref{fig:prompt-qa} (SearchQA).

\begin{figure}[ht]
\begin{tcolorbox}[
    title=Prompt Template for ALFWorld,
    fonttitle=\bfseries\small,
    colback=white,
    colframe=black!70,
    colbacktitle=black!80,
    coltitle=white,
    boxrule=0.5pt,
    arc=4pt,
    left=6pt, right=6pt, top=4pt, bottom=4pt
]
\small
You are an expert agent operating in the ALFRED embodied environment. Your task is to: \textcolor{promptorange}{\{task\_description\}}. Prior to this step, you have already taken \textcolor{promptorange}{\{step\_count\}} step(s). Below are the most recent \textcolor{promptorange}{\{history\_length\}} observations and the corresponding actions you took: \textcolor{promptorange}{\{action\_history\}}. You are now at step \textcolor{promptorange}{\{current\_step\}} and your current observation is: \textcolor{promptorange}{\{current\_observation\}}. Your admissible actions of the current situation are: [\textcolor{promptorange}{\{admissible\_actions\}}].

Now it's your turn to take an action. You should first reason step-by-step about the current situation. This reasoning process MUST be enclosed within \textcolor{promptred}{$<$think$>$ $<$/think$>$} tags. Once you've finished your reasoning, you should choose an admissible action for current step and present it within \textcolor{promptred}{$<$action$>$ $<$/action$>$} tags.
\end{tcolorbox}
\caption{Prompt template for ALFWorld.}
\label{fig:prompt-alfworld}
\end{figure}

\begin{figure}[ht]
    \centering
    \begin{tcolorbox}[
        title=Prompt Template for WebShop,
        fonttitle=\bfseries\small,
        colback=white,
        colframe=black!70,
        colbacktitle=black!80,
        coltitle=white,
        boxrule=0.5pt,
        arc=4pt,
        left=6pt, right=6pt, top=4pt, bottom=4pt
    ]
    \small
    You are an expert autonomous agent operating in the WebShop e-commerce environment. Your task is to: \textcolor{promptorange}{\{task\_description\}}. Prior to this step, you have already taken \textcolor{promptorange}{\{step\_count\}} step(s). Below are the most recent \textcolor{promptorange}{\{history\_length\}} observations and the corresponding actions you took: \textcolor{promptorange}{\{action\_history\}}. You are now at step \textcolor{promptorange}{\{current\_step\}} and your current observation is: \textcolor{promptorange}{\{current\_observation\}}. Your admissible actions for the current situation are: [\textcolor{promptorange}{\{available\_actions\}}].

    Now it's your turn to take one action for the current step. You should first reason step-by-step about the current situation, then think carefully which admissible action best advances the shopping goal. This reasoning process MUST be enclosed within \textcolor{promptred}{$<$think$>$ $<$/think$>$} tags. Once you've finished your reasoning, you should choose an admissible action for current step and present it within \textcolor{promptred}{$<$action$>$ $<$/action$>$} tags.
    \end{tcolorbox}
    \caption{Prompt template for WebShop.}
    \label{fig:prompt-webshop}
\end{figure}

\begin{figure}[ht]
    \centering
    \begin{tcolorbox}[
        title=Prompt Template for SearchQA,
        fonttitle=\bfseries\small,
        colback=white,
        colframe=black!70,
        colbacktitle=black!80,
        coltitle=white,
        boxrule=0.5pt,
        arc=4pt,
        left=6pt, right=6pt, top=4pt, bottom=4pt
    ]
    \small
    You are an expert agent tasked with answering the given question step-by-step. 
    
    Your question: \textcolor{promptorange}{\{task\_description\}}. 
    
    Prior to this step, you have already taken \textcolor{promptorange}{\{step\_count\}} step(s). Below is the interaction history where \textcolor{promptred}{$<$search$>$ $<$/search$>$} wrapped your past search queries and \textcolor{promptred}{$<$information$>$ $<$/information$>$} wrapped the corresponding search results returned by the external search engine.

    History: \textcolor{promptorange}{\{memory\_context\}}

    Now it's your turn to respond for the current step. You should first conduct reasoning process. This process MUST be enclosed within \textcolor{promptred}{$<$think$>$ $<$/think$>$} tags. After completing your reasoning, choose only one of the following actions (do not perform both):

    (1) If you find you lack some knowledge, you can call a search engine to get more external information using format: \textcolor{promptred}{$<$search$>$} your query \textcolor{promptred}{$<$/search$>$}.

    (2) If you have enough knowledge to answer the question confidently, provide your final answer within \textcolor{promptred}{$<$answer$>$ $<$/answer$>$} tags, without detailed illustrations. For example, \textcolor{promptred}{$<$answer$>$Beijing$<$/answer$>$}.
    \end{tcolorbox}
    \caption{Prompt template for SearchQA.}
    \label{fig:prompt-qa}
\end{figure}

\end{document}